\documentclass[10pt,twocolumn,letterpaper]{article}

\usepackage{authblk}
\usepackage{cvpr}              

\usepackage{graphicx}
\usepackage{amsmath}
\usepackage{amssymb}
\usepackage{booktabs}
\usepackage{multirow}
\usepackage{algorithm}
\usepackage{algorithmic}
\usepackage{xcolor}
\usepackage{graphicx}
\usepackage{arydshln}
\usepackage{makecell}
\usepackage{enumitem}
\usepackage{amsthm}
\usepackage[table]{colortbl}
\usepackage[pagebackref,breaklinks,colorlinks]{hyperref}

\usepackage[capitalize]{cleveref}
\crefname{section}{Sec.}{Secs.}
\Crefname{section}{Section}{Sections}
\Crefname{table}{Table}{Tables}
\crefname{table}{Tab.}{Tabs.}

\def\cvprPaperID{8329} 
\def\confName{CVPR}
\def\confYear{2023}

\begin{document}

\title{Robust Outlier Rejection for 3D Registration with Variational Bayes}
\author[1]{Haobo Jiang}
\author[2]{Zheng Dang}
\author[2]{Zhen Wei}
\author[1]{Jin Xie$^*$}
\author[1]{Jian Yang$^*$}
\author[2]{Mathieu Salzmann$^*$}
\affil[1]{PCA Lab, Nanjing University of Science and Technology, China}
\affil[2]{CVLab, EPFL, Switzerland} 
\affil[ ]{\tt\small {\{jiang.hao.bo, csjxie, csjyang\}@njust.edu.cn}}
\affil[ ]{\tt\small {\{zheng.dang, zhen.wei, mathieu.salzmann\}@epfl.ch}}

\maketitle

\let\thefootnote\relax\footnotetext{$^*$Corresponding authors}
\let\thefootnote\relax\footnotetext{Haobo Jiang, Jin Xie, and Jian Yang are with PCA Lab, Key Lab of Intelligent Perception and Systems for High-Dimensional Information of Ministry of Education, and Jiangsu Key Lab of Image and Video Understanding for Social Security, School of Computer Science and Engineering, Nanjing University of Science and Technology, China.}


\newtheorem{theorem}{Theorem}
\newtheorem{lemma}{Lemma}
\newtheorem{property}{Property}
\newtheorem{definition}{Definition}
\newtheorem{example}{Example}
\begin{abstract} 
Learning-based outlier (mismatched correspondence) rejection for robust 3D registration generally formulates the outlier removal as an inlier/outlier classification problem. The core for this to be successful is to learn the discriminative inlier/outlier feature representations. In this paper, we develop a novel variational non-local network-based outlier rejection framework for robust alignment. By reformulating the non-local feature learning with variational Bayesian inference, the Bayesian-driven long-range dependencies can be modeled to aggregate discriminative geometric context information for inlier/outlier distinction. Specifically, to achieve such Bayesian-driven contextual dependencies, each query/key/value component in our non-local network predicts a prior feature distribution and a posterior one. Embedded with the inlier/outlier label, the posterior feature distribution is label-dependent and discriminative. Thus, pushing the prior to be close to the discriminative posterior in the training step enables the features sampled from this prior at test time to model high-quality long-range dependencies. Notably, to achieve effective posterior feature guidance, a specific probabilistic graphical model is designed over our non-local model, which lets us derive a variational low bound as our optimization objective for model training. Finally, we propose a voting-based inlier searching strategy to cluster the high-quality hypothetical inliers for transformation estimation. Extensive experiments on 3DMatch, 3DLoMatch, and KITTI datasets verify the effectiveness of our method.  
Code is available at \href{https://github.com/Jiang-HB/VBReg}{https://github.com/Jiang-HB/VBReg}.
\end{abstract}

\section{Introduction}
\label{sec:intro}
Point cloud registration is a fundamental but challenging 3D computer vision task, with many potential applications such as 3D scene reconstruction~\cite{agarwal2011building,schonberger2016structure}, object pose estimation~\cite{wong2017segicp,dang2022learning}, and Lidar SLAM ~\cite{deschaud2018imls,zhang2014loam}. It aims to align two partially overlapping point clouds by estimating their relative rigid transformation, i.e., 3D rotation and 3D translation. A popular approach to address the large-scale scene registration problem consists of extracting point descriptors~\cite{zeng20173dmatch,choy2019fully,deng2018ppf,frome2004recognizing,rusu2009fast,salti2014shot} and establishing correspondences between the two point clouds, from which the transformation can be obtained geometrically. In this context, much effort has been dedicated to designing traditional and deep learning-based descriptors~\cite{zeng20173dmatch,choy2019fully,wang2021you,bai2020d3feat,huang2021predator}. 
However, the resulting correspondences inevitably still suffer from outliers (wrong matchings), particularly in challenging cases, such as low-overlap, repetitive structures, or noisy point sets, leading to registration failure.

To address this, many outlier filtering strategies have been developed to robustify the registration process.  
These include traditional rejection methods using random sample consensus~\cite{fischler1981random}, point-wise descriptor similarity~\cite{lowe2004distinctive,bradski2000opencv} or group-wise spatial consistency~\cite{yang2019performance}. 
Deep learning methods have also been proposed, focusing on learning correspondence features used to estimate inlier confidence values~\cite{choy2020deep,pais20203dregnet,bai2021pointdsc}.
In particular, the current state-of-the-art method, PointDSC~\cite{bai2021pointdsc}, relies on a spatial consistency-driven non-local network to capture long-range context in its learned correspondence features. 
While effective, PointDSC still yields limited registration robustness, particularly for scenes with a high outlier ratio, where the spatial consistency constraints may become ambiguous~\cite{quan2020compatibility}, thereby degrading the correspondence features' quality.

In this paper, we propose to explicitly account for the ambiguities arising from high outlier ratios by developing a probabilistic feature learning framework. 
To this end, we introduce a variational non-local network based on an attention mechanism to learn discriminative inlier/outlier feature representations for robust outlier rejection. 
Specifically, to capture the
ambiguous nature of long-range contextual dependencies, we inject a random feature in each query, key, and value component in our non-local network. The prior/posterior distributions of such random features are predicted by prior/posterior encoders. 
To encourage the resulting features to be discriminative, we make the posterior feature distribution label-dependent. During training, we then push the prior distribution close to the label-dependent posterior, thus allowing the prior encoder to also learn discriminative query, key, and value features. This enables the features sampled from this prior at test time to model high-quality long-range dependencies.

To achieve effective variational inference, we customize a probabilistic graphical model over our variational non-local network to characterize the conditional dependencies of the random features. This lets us derive a variational lower bound as the optimization objective for network training. 
Finally, we propose a voting-based deterministic inlier searching mechanism for transformation estimation, where the correspondence features learned from all non-local iterations jointly vote for high-confidence \textit{hypothetical inliers} for SVD-based transformation estimation. 
We theoretically analyze the robustness of our deterministic inlier searching strategy compared to RANSAC, which also motivates us to design a conservative seed selection mechanism to improve robustness in sparse point clouds. 

To summarize, our contributions are as follows:
\begin{itemize}
	\item We propose a novel variational non-local network for outlier rejection, learning discriminative correspondence features with Bayesian-driven long-range contextual dependencies. 
	\item We customize the probabilistic graphical model over our variational non-local network and derive the variational low bound for effective model optimization. 
	\item We introduce a Wilson score-based voting mechanism to search high-quality \textit{hypothetical inliers}, and theoretically demonstrate its superiority over RANSAC. 
	\end{itemize}
Our experimental results on extensive benchmark datasets demonstrate that our framework outperforms the state-of-the-art registration methods. 

\section{Related Work}

\noindent\textbf{End-to-end Registration Methods.}  
With the advances of deep learning in the 3D vision field  \cite{qi2017pointnet}, the learning-based end-to-end registration model has achieved increasing research attention. 
DCP~\cite{wang2019deep} uses the feature similarity to establish pseudo correspondences for SVD-based transformation estimation. 
RPM-Net~\cite{yew2020rpm} exploits the Sinkhorn layer and annealing for discriminative matching map generation.  
\cite{jiang2021sampling,jiang2021planning} integrate the cross-entropy method into the deep model for robust  registration. 
RIENet~\cite{shen2022reliable} uses the structure difference between the source neighborhood and the pseudo-target one for inlier confidence evaluation.
With the powerful feature representation of Transformer, RegTR~\cite{yew2022regtr}  effectively aligns large-scale indoor scenes in an end-to-end manner. 
\cite{dang2022learning} propose a match-normalization layer for robust registration in the real-world 6D object pose estimation task.
More end-to-end models such as \cite{li2020iterative,choy2020deep,pais20203dregnet,li2020unsupervised,li2019pc,zhu2020reference,fu2021robust} also present impressive precisions.

\noindent\textbf{Learning-based Feature Descriptors.} 
To align the complex scenes, a popular pipeline is to exploit feature descriptors for 3D matching. 
Compared to hand-crafted descriptors such as \cite{frome2004recognizing,rusu2009fast,salti2014shot}, the deep feature descriptor presents superior registration precision and has achieved much more attention in recent years. The pioneering 3DMatch \cite{zeng20173dmatch} exploits the Siamese 3D CNN to learn the local geometric feature via contrastive loss.  
FCGF \cite{choy2019fully} exploits a fully convolutional network for dense feature extraction in a one-shot fashion. 
Furthermore, D3feat \cite{bai2020d3feat} jointly learns the dense feature descriptor and the detection score for each point. 
By integrating the overlap-attention module into D3feat, Predator \cite{huang2021predator} largely improves the registration reliability in low-overlapping point clouds. 
YOHO \cite{wang2021you} utilizes the group equivariant feature learning to achieve the rotation invariance and shows great robustness to the point density and the noise interference.
\cite{qin2022geometric} develops a geometric transformer to learn the geometric context for robust super-point matching. Lepard~\cite{li2022lepard} embeds the relative 3D positional encoding into the transformer for discriminative descriptor learning. 

\noindent\textbf{Outlier Rejection Methods.}
Despite significant progress in learning-based feature descriptor, generating mismatched correspondences (outliers) in some challenging scenes remains unavoidable.
Traditional outlier filtering methods, such as RANSAC \cite{fischler1981random} and its variants \cite{le2019sdrsac, li2020gesac, barath2018graph}, use repeated sampling and verification for outlier rejection. However, these methods tend to have a high time cost, particularly in scenes with a high outlier ratio. Instead, FGR \cite{zhou2016fast} and TEASER \cite{yang2020teaser} integrate the robust loss function into the optimization objective to weaken the interference from outliers. Recently, 
Chen et al. \cite{chen2022sc2} developed second-order spatial compatibility for robust consensus sampling. 
With the rise of deep 3D vision, most learnable outlier rejection models \cite{choy2020deep, pais20203dregnet} formulate outlier rejection as a binary classification task and reject correspondences with low confidence. Yi et al. \cite{yi2018learning} proposed a context normalization-embedded deep network for inlier evaluation, while Brachmann et al. \cite{brachmann2019neural} enhanced classical RANSAC with neural-guided prior confidence. As our baseline, PointDSC \cite{bai2021pointdsc} proposes exploiting a spatial consistency-guided non-local inlier classifier for inlier evaluation, followed by neural spectral matching for robust registration. However, under high outlier ratios, spatial consistency can be ambiguous (as shown in Fig.~\ref{ratio}), misleading non-local feature aggregation. Instead, we propose exploiting Bayesian-driven long-range dependencies for discriminative non-local feature learning.

\section{Approach}
\label{sec1}
\subsection{Background}
\noindent\textbf{Problem Setting.}
In the pairwise 3D registration task, given a source point cloud $\mathbf{X}=\{\mathbf{x}_i\in\mathbb{R}^3\mid i=1,...,|\mathbf{X}|\}$ and a target point cloud $\mathbf{Y}=\{\mathbf{y}_j\in\mathbb{R}^3\mid j=1,...,|\mathbf{Y}|\}$, we aim to find their optimal rigid transformation consisting of a rotation matrix ${\mathbf{R}^*} \in SO(3)$ and a translation vector ${\mathbf{t}^*} \in \mathbb{R}^3$ to align their overlapping region precisely. 
In this work, we focus on the descriptor-based pipeline for large-scale scene registration. 
Based on the feature-level nearest neighbor, we construct a set of putative correspondences $\mathcal{C}=\left\{\mathbf{c}_i=\left(\mathbf{x}_i, \mathbf{y}_i\right)\in\mathbb{R}^6\mid i=1,...,|\mathcal{C}|\right\}$. 
The inlier (correctly matched correspondence) is defined as the correspondence satisfying $\left\|\mathbf{R}^*\mathbf{x}_{i} + \mathbf{t}^*-\mathbf{y}_{i}\right\|<\varepsilon$, where $\varepsilon$ indicates the inlier threshold. 


\noindent\textbf{Vanilla Non-local Feature Embedding.} Given the putative correspondence set $\mathcal{C}$, \cite{bai2021pointdsc} leverages the spatial consistency-guided non-local network (\textit{SCNonlocal}) for their feature embedding. 
The injected geometric compatibility matrix can effectively regularize the long-range dependencies for discriminative inlier/outlier feature learning. 
In detail, it contains $L$ iterations and the feature aggregation in $l$-th iteration can be  formulated as: 
\begin{equation}\small \label{nonlocal_op}
	\setlength{\abovedisplayskip}{2pt}
	\setlength{\belowdisplayskip}{2pt}
	\begin{split}
		\mathbf{F}^{l+1}_{i} = \mathbf{F}^{l}_{i} + \operatorname{MLP}\Big(\sum_{j=1}^{|\mathcal{C}|}\operatorname{softmax}_j(\boldsymbol{\alpha}^{l}\boldsymbol{\beta})\mathbf{V}_j^l\Big),
	\end{split}
\end{equation}
where $\mathbf{F}^{l}_{i}\in\mathbb{R}^{d}$ indicates the  feature embedding of correspondence $\mathbf{c}_i$ in $l$-th iteration (the initial feature $\mathbf{F}^{0}_{i}$ is obtained via linear projection on $\mathbf{c}_i$) and $\mathbf{V}^l_i =f^l_v(\mathbf{F}^{l}_{i})\in\mathbb{R}^{d}$ is the projected value feature. $\boldsymbol{\alpha}^{l}\in\mathbb{R}^{|\mathcal{C}|\times|\mathcal{C}|}$ is the non-local attention map whose entry $\boldsymbol{\alpha}_{i,j}^{l}$ reflects the feature similarity between the projected query feature $\mathbf{Q}^l_i = f^l_q(\mathbf{F}^{l}_{i})\in\mathbb{R}^d$ and the key feature $\mathbf{K}^l_i =f^l_k(\mathbf{F}^{l}_{j})\in\mathbb{R}^d$. $\boldsymbol{\beta}\in\mathbb{R}^{|\mathcal{C}|\times|\mathcal{C}|}$ represents the geometric compatibility matrix of correspondences, where the compatibility between $\mathbf{c}_i$ and $\mathbf{c}_j$ is:
\begin{equation}\label{compat1}\small
	\setlength{\abovedisplayskip}{2pt}
	\setlength{\belowdisplayskip}{2pt}
	\begin{split}
		\boldsymbol{\beta}_{i,j} = \max\Big(0, 1 - \frac{d^2_{ij}}{\varepsilon^2}\Big), \ \ d_{ij} = \left| \|\mathbf{x}_i-\mathbf{x}_j\| - \|\mathbf{y}_i-\mathbf{y}_j\|  \right|. 
	\end{split}
\end{equation} 
Based on the fact that the geometric distance $d_{i,j}$  of inliers $\mathbf{c}_i$ and $\mathbf{c}_j$ tend to be minor, Eq.~\ref{compat1} will assign a high compatibility value on the inlier pair, thereby promoting the non-local network to effectively cluster the inlier features for discriminative inlier/outlier feature learning. 

\begin{figure}
	\centering
	\includegraphics[width=0.9\columnwidth]{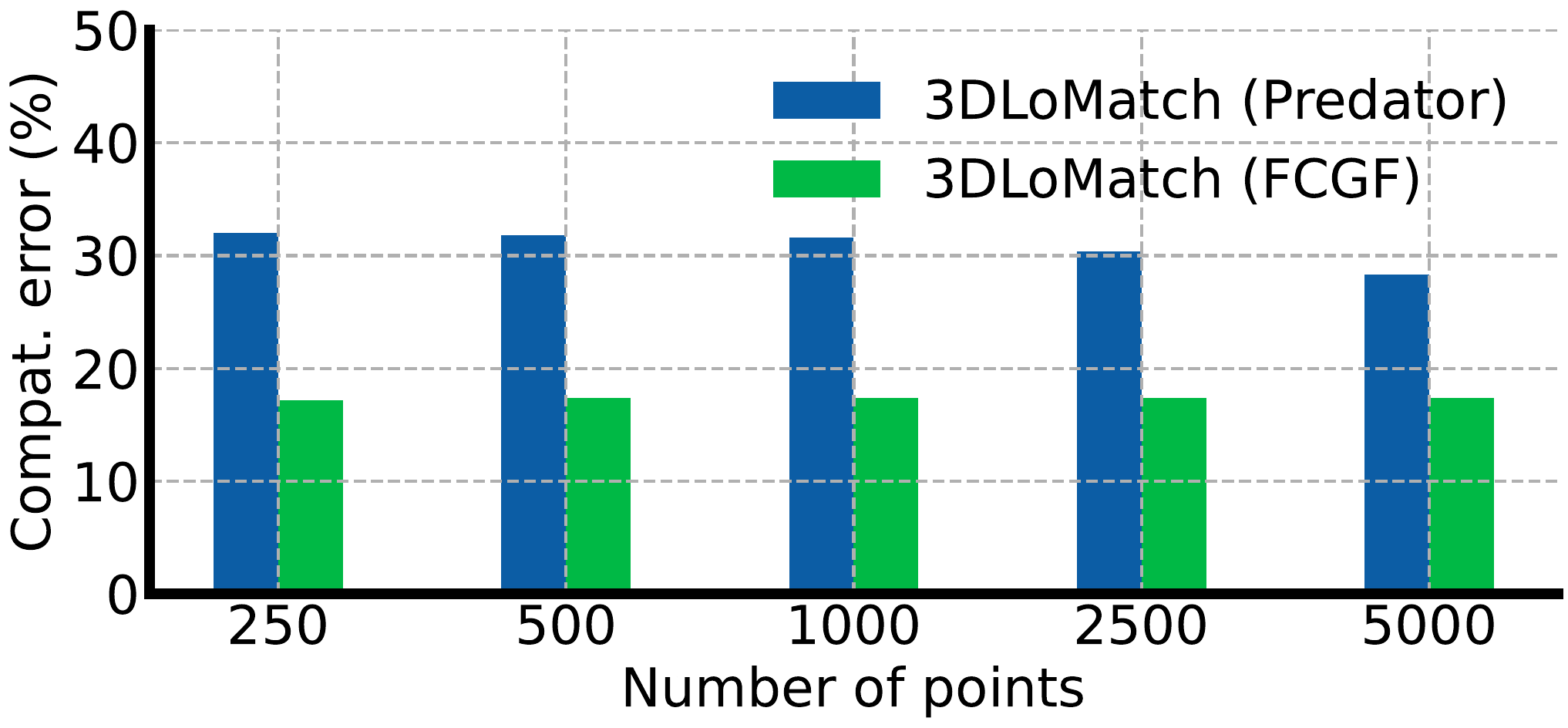}
	\vspace{-4mm}
	\caption{The ratio of inlier-outlier pairs with positive compatibilities in {3DLoMatch}~\cite{huang2021predator} using {FCGF}   and {Predator} descriptors.} \label{fig:model}
	\label{ratio}
	\vspace{-4mm}
\end{figure}

\subsection{Variational Non-local  Feature Embedding}  \label{vbnonlocal}
While effective, \textit{SCNonlocal} still suffers from ambiguous long-range dependencies, especially in some challenging scenes (e.g., the low-overlapping case). 
Two essential reasons are: \textbf{(i)} Wrong geometric compatibility. As shown in Fig.~\ref{ratio}, for {3DLoMatch} dataset with {Predator} and FCGF descriptors, almost 30\% and 17\% of inlier-outlier pairs own the positive compatibility values, respectively, which potentially misleads the attention weight for wrong feature clustering. 
\textbf{(ii)} Lack of uncertainty modeling. In symmetric or repetitive scenes, the inlier/outlier prediction contains significant uncertainty. Therefore, it's necessary to design a robust feature representation to capture such uncertainty. 




To overcome them, we develop a variational non-local network, a probabilistic feature learning framework, for discriminative correspondence embedding. 
Our core idea is to inject random features into our model to capture the ambiguous nature of long-range dependencies, and then leverage the variational Bayesian inference to model the Bayesian-driven long-range dependencies for discriminative feature aggregation. 
Specifically, we first introduce the random feature variables ${z}^l_{k,i}$, ${z}^l_{q,i}$ and ${z}^l_{v,i}$ into the key $\mathbf{K}^l_{i}$, query $\mathbf{Q}^l_{i}$ and value $\mathbf{V}^l_{i}$ components in our non-local module to capture their potential uncertainty in the long-range dependency modeling. 
Then, the prior/posterior encoders are constructed to predict their prior feature distribution and the posterior one, respectively. 
Embedded with the inlier/outlier label, the posterior feature distribution is label-dependent and discriminative. \
Thus, by pushing the prior close to the discriminative posterior in the training phase, this prior at test time also tends to sample discriminative query, key, and value features for high-quality long-range dependency modeling. 

\noindent\textbf{Probabilistic Graphical Model over Variational Non-local Network.} 
To achieve effective variational Bayesian inference, we need to first characterize the conditional dependencies of the injected random features so that the variational lower bound can be derived as the optimization objective for model training. 
As shown in Fig.~\ref{graph}, we customize the probabilistic graphical model over our non-local network to clarify the dependencies of random features (the circles). 
The solid line denotes the label prediction process, while the dashed line represents our label-dependent posterior encoder (\ie, inference model).
Notably, the deterministic hidden query/key/value features $ \mathbf{h}^{l}_{k,i}\in\mathbb{R}^{d'}$, $ \mathbf{h}^{l}_{q,i}\in\mathbb{R}^{d'}$, and $ \mathbf{h}^{l}_{v,i}\in\mathbb{R}^{d'}$ are also introduced to summarize the historical information for better feature updating in each iteration.  

\noindent\textbf{Inlier/outlier Prediction Process.} 
Based on the defined conditional dependencies in Fig.~\ref{graph}, the prediction process of correspondence labels $\mathbf{b}=\left\{b_1, b_2, ..., b_{|\mathcal{C}|} \mid b_i \in \{0,1\}\right\}$ (1 indicates inlier and 0 outlier) is formulated as follows. 
Beginning with the initial linear projection $\tilde{\mathbf{F}}^0\in\mathbb{R}^{|\mathcal{C}|\times d}$ of correspondences $\mathcal{C}=\{\mathbf{c}_i\}$, we iteratively perform the probabilistic non-local aggregation for feature updating. 
In the $l$-th iteration, we first employ a Gated Recurrent Unit (GRU)~\cite{chung2014empirical} to predict the hidden query/key/value features which summarize the historical query/key/value features (sampled from the prior distributions) and the correspondence features in previous iterations: 
\begin{equation}\label{compat}
	\setlength{\abovedisplayskip}{3pt}
	\setlength{\belowdisplayskip}{3pt}
	\begin{split}
		 \mathbf{h}^{l}_{q,i}=\operatorname{GRU}_q(\mathbf{h}^{l-1}_{q,i}, [\mathbf{z}^{l-1}_{q,i}, \tilde{\mathbf{F}}_i^{l-1}]), \\ 
		 \mathbf{h}^{l}_{k,i}=\operatorname{GRU}_k(\mathbf{h}^{l-1}_{k,i}, [\mathbf{z}^{l-1}_{k,i}, \tilde{\mathbf{F}}_i^{l-1}]), \\ 
		 \mathbf{h}^{l}_{v,i}=\operatorname{GRU}_v(\mathbf{h}^{l-1}_{v,i}, [\mathbf{z}^{l-1}_{v,i}, \tilde{\mathbf{F}}_i^{l-1}]), \\ 
	\end{split}
\end{equation} 
where $[\cdot,\cdot]$ denotes the feature concatenation and $\tilde{\mathbf{F}}_i^{l-1}$ is the learned correspondence features of $\mathbf{c}_i$ in iteration $l-1$. 
Then, with as input the predicted hidden features, the prior encoder $p_\theta(\cdot)$ is utilized to predict the prior feature distribution for query/key/value, respectively. 
Furthermore, we sample features $\mathbf{z}^{l}_{q,i}\in\mathbb{R}^{\tilde{d}}$, $\mathbf{z}^{l}_{k,i}\in\mathbb{R}^{\tilde{d}}$ and $\mathbf{z}^{l}_{v,i}\in\mathbb{R}^{\tilde{d}}$ from the predicted prior query/key/value distribution and combine them with the hidden features to predict the corresponding query $\tilde{\mathbf{Q}}_i^l\in\mathbb{R}^d$, key $\tilde{\mathbf{K}}_i^l\in\mathbb{R}^d$ and value $\tilde{\mathbf{V}}_i^l\in\mathbb{R}^d$ through a neural network $f^{q,k,v}_\theta: \mathbb{R}^{d'+\tilde{d}}\rightarrow\mathbb{R}^d$:
\begin{equation}\label{compat}
	\begin{split}
			\mathbf{z}^{l}_{q,i}&\sim p_\theta({z}^{l}_{q,i}\mid \mathbf{h}^l_{q,i}),\ \  \tilde{\mathbf{Q}}_i^l=f^q_\theta(\left[\mathbf{z}^{l}_{q,i},  \mathbf{h}^{l}_{q,i}\right]),\\ 
\mathbf{z}^{l}_{k,i}&\sim p_\theta({z}^{l}_{k,i}\mid \mathbf{h}^l_{k,i}),\ \  \tilde{\mathbf{K}}_i^l=f^k_\theta(\left[\mathbf{z}^{l}_{k,i},  \mathbf{h}^{l}_{k,i}\right]),\\ 
		\mathbf{z}^{l}_{v,i}&\sim p_\theta({z}^{l}_{v,i}\mid \mathbf{h}^l_{v,i}),\ \  \tilde{\mathbf{V}}_i^l=f^v_\theta(\left[\mathbf{z}^{l}_{v,i},  \mathbf{h}^{l}_{v,i}\right]), 
	\end{split}
\end{equation} 
where the prior feature distribution is the Gaussian distribution with the mean and the standard deviation parameterized by a neural network. 
Finally, with the learned $\tilde{\mathbf{Q}}_i^l$, $\tilde{\mathbf{K}}_i^l$ and $\tilde{\mathbf{V}}_i^l$, the correspondence feature $\tilde{\mathbf{F}}_i^l$ in $l$-th iteration can be aggregated with the same non-local operation in Eq.~\ref{nonlocal_op}. After $L$ feature iterations, we feed the correspondence feature $\tilde{\mathbf{F}}_i^L$ in the last iteration into a label prediction model $y_{\theta}$ to predict the inlier/outlier labels ${b}_i\sim y_{\theta}({b}_i\mid \tilde{\mathbf{F}}_i^L)$, where the label prediction model outputs a scalar Gaussian distribution with the mean parameterized by the neural network and the unit variance. 

\begin{figure}
	\centering
	\includegraphics[width=\columnwidth]{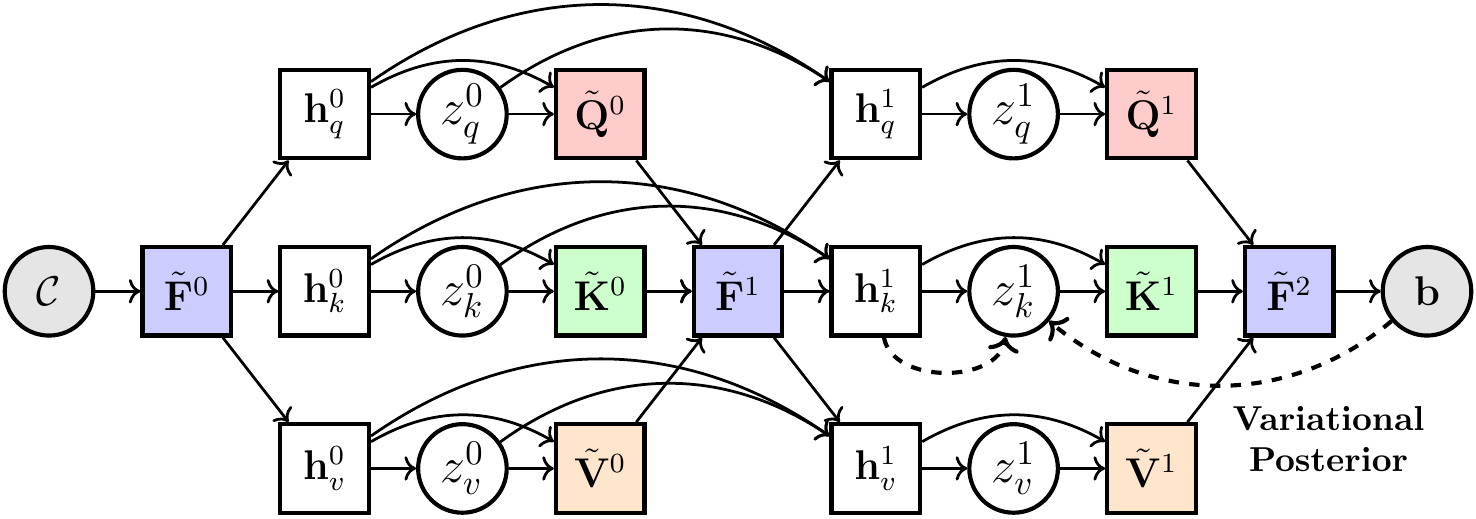}
	\vspace{-4mm}
	\caption{Probabilistic graphical model for our variational non-local network.  For simplicity, we just demonstrate two iterations. The white circles indicate the random features and the white squares denote the deterministic hidden features. 
		The solid line represents the inlier/outlier prediction process and the dashed line denotes the label-dependent variational posterior encoder. 
		We just show the variational posterior for ${z}^1_k$.} \label{fig:model}
	\label{graph}
	\vspace{-4mm}
\end{figure}

\noindent\textbf{Variational Posterior Encoder.}  
Due to the nonlinearity of our variational non-local model, we cannot directly derive the precise posterior distribution for random query/key/value features using the standard Bayes' theorem. Taking inspiration from the Variational Bayesian inference, we construct a label-dependent posterior encoder $q_\phi(\cdot)$ to to approximate the feature posterior:
\begin{equation}\label{compat}
	\begin{split}
		\mathbf{z}^{l}_{q,i}&\sim q_\phi({z}^{l}_{q,i}\mid [\mathbf{h}^l_{q,i}, [{b}_i]_{\times k}])\\ 
		\mathbf{z}^{l}_{k,i}&\sim q_\phi({z}^{l}_{k,i}\mid [\mathbf{h}^l_{k,i}, [{b}_i]_{\times k}])\\ 
		\mathbf{z}^{l}_{v,i}&\sim q_\phi({z}^{l}_{v,i}\mid [\mathbf{h}^l_{v,i}, [{b}_i]_{\times k}]),
	\end{split}
\end{equation} 
where $[{b}_i]_{\times k}$ indicates a label vector generated by tiling the scalar label $k$ times. The output of each posterior encoder is a diagonal Gaussian distribution with parameterized mean and standard deviation.


\noindent\textbf{Variational Lower Bound.} Finally, we derive the optimization objective $\operatorname{ELBO}(\theta, \phi)$, the variational (evidence) lower bound of log-likelihood correspondence labels $\ln y_\theta(\mathbf{b} \mid \mathcal{C})$,  to train our variational non-local network (Please refer to Appendix~\ref{proof} for the detailed derivation):
\begin{equation}\label{lb}\small
	\begin{split}
		&\operatorname{ELBO}(\theta, \phi) =  \mathbb{E}_{\prod^{\textcolor{black}{L-1}}_{l=\textcolor{black}{0}}{q_\phi(z^{l}_{q,k,v}\mid {\mathbf{h}^l_{q,k,v}}}, \mathbf{b})}\left[\ln y_\theta(\mathbf{b}\mid {\textcolor{black}{\tilde{\textcolor{black}{\mathbf{F}}}}^L})\right] - 
		\\ &\sum^{\textcolor{black}{L-1}}_{l=\textcolor{black}{0}}\textcolor{black}{\textcolor{black}{\mathbb{E}_{q_\phi}}} \left[\operatorname{D_{KL}}\left(\textcolor{black}{q_\phi(\textcolor{black}{z}^{l}_{q,k,v}\mid {\mathbf{h}^l_{q,k,v}}, \mathbf{b})}|| p_\theta(\textcolor{black}{z}^{l}_{q,k,v}\mid {{\mathbf{h}^l_{q,k,v}}})\right)\right] 
	\end{split}
\end{equation}
where for clarity, we utilize the subscript $q,k,v$ to denote the same operator performed on query/key/value. 
 $\operatorname{D_{KL}}(\cdot||\cdot)$ denotes the Kullback–Leibler (KL) divergence between two distributions. 
By maximizing the variational lower bound above, we can optimize the network parameters to indirectly maximize the log-likelihood value of correspondence labels. 
Eq.~\ref{lb}  indicates that the discriminative, label-dependent feature posterior explicitly constrains the prior by reducing their KL divergence in the training phase, which promotes the query, key, and value features sampled from the prior to model the high-quality long-term dependencies at test time.

\subsection{Voting-based Inlier Searching} 
With the learned correspondence features above, we then propose a voting-based sampling strategy to search the desired inlier subset from the entire putative correspondences for optimal transformation estimation. 
Our sampling mechanism is deterministic and efficient. 
We first select a set of high-confidence seeds $\mathcal{C}_{seed}$ from $\mathcal{C}$ based on their inlier confidence (predicted in  \S~\ref{vbnonlocal}) and the Non-Maximum Suppression (as performed in~\cite{bai2021pointdsc}), where the number of seeds $|\mathcal{C}_{seed}|=\lfloor|C| * v\rfloor$ ($v$ is a fixed seed ratio). 
Then, for each seed, we cluster its most compatible correspondences into it to form the \textit{hypothetical inliers}. 
Ideally, if the seed is an inlier and the compatible value is correct (without ambiguity problem as in Fig.~\ref{fig:model}), its most compatible correspondences are also inliers, thus we can cluster the desired inlier subset successfully. 
To cluster high-confidence \textit{hypothetical inliers}, we perform a coarse-to-fine clustering strategy, including the voting-based coarse clustering and the Wilson score-based fine clustering. 

\noindent\textbf{Voting-based Coarse Clustering.} 
We view the correspondence features $\{\tilde{\mathbf{F}}^1,...,\tilde{\mathbf{F}}^L\}$ learned from all non-local iterations as $L$ voters. 
For $l$-th voter, it deterministically clusters $\kappa-1$  most compatible correspondences ${}^{(\kappa-1)}\mathcal{A}^l_{\mathbf{c}_i}$ for each seed $\mathbf{c}_i\in\mathcal{C}_{seed}$ to from the \textit{hypothetical inliers} ${}^{(\kappa)}\mathcal{M}^l_{\mathbf{c}_i}=\{\mathbf{c}_i\}\cup{}^{(\kappa-1)}\mathcal{A}^l_{\mathbf{c}_i}$. 
To this end, 
we first compute the compatibility matrix $\mathbf{S}^l \in \mathbb{R}^{|\mathcal{C}|\times|\mathcal{C}|}$ for $l$-th voter, where each entry $\mathbf{S}^l_{i,j}=\operatorname{clip}(1-(1-cos(\tilde{\mathbf{F}}_i^{l}, \tilde{\mathbf{F}}_j^{l})/\sigma^2), 0, 1)\cdot{\boldsymbol{\beta}}_{i,j}$, where the parameter $\sigma$ is to control the feature compatibility and geometric compatibility ${\boldsymbol{\beta}}_{i,j}$ is defined in Eq.~\ref{compat1}. 
Thus, ${}^{(\kappa-1)}\mathcal{A}^l_{\mathbf{c}_i}$ can be determined as $\{\mathbf{c}_j\mid \mathbf{S}^l_{i,j} \geq \gamma_{\kappa-1} \}$, where threshold $\gamma_{\kappa-1}$ is the $(\kappa-1)$-th largest compatibility. 
Benefitting from our discriminative feature representations, $\mathbf{S}^l$ can effectively handle the ambiguity problem in Fig.~\ref{fig:model} and promote robust inlier clustering. 
After the voting process above, each seed $\mathbf{c}_i$ can achieve $L$  candidate \textit{hypothetical inliers} $\left\{{}^{(\kappa)}\mathcal{M}^1_{\mathbf{c}_i},...,{}^{(\kappa)}\mathcal{M}^L_{\mathbf{c}_i}\right\}$. 
%
%

\noindent\textbf{Wilson score-based Fine Clustering.} 
Based on the voted \textit{hypothetical inliers} above, we then fuse these voted results with Wilson score~\cite{wilson1927probable} to form the high-quality \textit{hypothetical inliers} ${}^{(\kappa)}\tilde{\mathcal{M}}_{\mathbf{c}_i}$. 
In detail. we denote $\mathcal{I}^l_{i,j}=\mathbb{I}\{\mathbf{c}_j \in {}^{(\kappa)}\mathcal{M}^l_{\mathbf{c}_i}\}$ to indicate 
whether ${}^{(\kappa)}\mathcal{M}^l_{\mathbf{c}_i}$ contains $\mathbf{c}_j$. 
The Wilson score $\mathbf{W}^n_{i,j}$ of accepting $\mathbf{c}_j$ into ${}^{(\kappa)}\tilde{\mathcal{M}}_{\mathbf{c}_i}$ can be computed as:
\begin{equation}\label{wilson}\small
	\setlength{\abovedisplayskip}{2pt}
	\setlength{\belowdisplayskip}{2pt}
	\begin{split}
		\mathbf{W}^n_{i,j} 
		= \frac{1}{1+\frac{z^2}{n}}\Bigg[\hat{p}^{(n)}_{i,j}+\frac{z^2}{2 n}-z\sqrt{\frac{\hat{p}^{(n)}_{i,j}(1-\hat{p}^{(n)}_{i,j})}{n}+\frac{z^2}{4 n^2}}\Bigg],
	\end{split}
\end{equation} 
where $\hat{p}^{(n)}_{i,j}=\frac{1}{n}\sum^n_{\tau=1}\mathcal{I}^\tau_{i,j}$ is the average acceptation ratio of top $n$ voters and  $z=1.96$ (\ie, z-score at 95\% confidence level). 
Eq.~\ref{wilson} indicates that the Wilson score not only considers the sample mean but also the confidence (positively related to sample number $n$). 
Finally, among the set of Wilson scores under different sample numbers $\{\mathbf{W}^1_{i,j},...,\mathbf{W}^L_{i,j} \}$, we choose the largest one as the final Wilson score $\tilde{\mathbf{W}}_{i,j}=\underset{n\in\{1,...,L\}}{\max}\mathbf{W}^n_{i,j}$ for $\mathbf{c}_j$. 
Thus, the final \textit{hypothetical inliers} set of $\mathbf{c}_i$ can be determined as ${}^{(\kappa)}\tilde{\mathcal{M}}_{\mathbf{c}_i}=\{\mathbf{c}_i\}\cup\{\mathbf{c}_j \mid \tilde{\mathbf{W}}_{i,j} \geq \gamma'_{\kappa-1}\}$, where  $\gamma'_{\kappa-1}$ is the ($\kappa-1$)-th largest Wilson score. 

\noindent\textbf{Theoretical Analysis and Conservative Seed Selection.} 
Finally, we try to theoretically analyze our deterministic inlier searching mechanism compared to RANSAC~\cite{fischler1981random} and further propose a simple but effective conservative seed selection strategy for more robust 3D registration in sparse point clouds. 
We let $\{{}^{(\kappa)}\mathcal{M}^{sac}_i\}_{i=1}^J$ be the randomly sampled \textit{hypothetical inlier} subset in RANSAC, and let  $\mathcal{C}_{in}$, $\mathcal{C}_{out}$ and $p_{in}$ be the inlier subset, outlier subset and the inlier ratio, respectively, $\mathcal{C}=\mathcal{C}_{in}\cup\mathcal{C}_{out}$, $p_{in}={|\mathcal{C}_{in}|}/{|\mathcal{C}|}$. 
We also denote the inliers in seeds as $\tilde{\mathcal{C}}_{in}=\mathcal{C}_{in}\cap\mathcal{C}_{seed}$. Then, we can derive the following theorem (Please refer to Appendix~\ref{proof2} for the derivation process). 
\begin{theorem}  \label{theorem1}
	Assume the number of outliers in ${}^{(\kappa)}\tilde{\mathcal{M}}_{\mathbf{c}_i}$ ($\mathbf{c}_i\in\tilde{\mathcal{C}}_{in}$) follows a Poisson distribution $Pois\left(\alpha \cdot \kappa\right)$. Then, if $\alpha < -\frac{1}{\kappa} \cdot \log \left[1 - (1 - p_{in}^\kappa)^{J/|\tilde{\mathcal{C}}_{in}|}\right]\triangleq\mathcal{U}$, the probability of our method achieving the inlier subset is greater than or equal to that of RANSAC. 
	\begin{equation}
			\setlength{\abovedisplayskip}{2pt}
		\setlength{\belowdisplayskip}{2pt}
		\begin{split}\label{neq1}
			&P\Big(\max_{\mathbf{c}_i \in \mathcal{C}_{seed}} |{}^{(\kappa)}\tilde{\mathcal{M}}_{\mathbf{c}_i}\cap \mathcal{C}_{in}| =\kappa\Big) \geq \\ &P\Big(\max_{1\leq i\leq J} |{}^{(\kappa)}\mathcal{M}^{sac}_i\cap \mathcal{C}_{in}| =\kappa\Big).
		\end{split}
	\end{equation}
\end{theorem} 
Theorem~\ref{theorem1} shows that as the inlier ratio $p_{in}$ degrades, the upper bound $\mathcal{U}$ will increase, indicating our method tends to be more likely to achieve the inlier subset than RANSAC in high outlier cases. 
Therefore, our method tends to be more robust to outliers than RANSAC. 
In addition, Theorem~\ref{theorem1}  also shows that the more inliers $|\tilde{\mathcal{C}}_{in}|$ in $\mathcal{C}_{seed}$, the better for our method. 
It means selecting all correspondences in $\mathcal{C}$ as seeds theoretically seems the best choice since it can avoid missing any inlier. 
However, in real implementation, we have to perform inlier selection to accelerate the registration speed. 
Nevertheless, in sparse point clouds, the number of correspondences is small. 
If we still perform the original inlier selection, fewer seeds even may not contain any inlier, leading to registration failure. 
To overcome it, we develop a conservative seed selection strategy, which changes the seed number to $\max\{\lfloor v\cdot|\mathcal{C}|\rfloor, n\} (n < |\mathcal{C}|)$, where the lower bar $n$ of the seed number can effectively avoid selecting too fewer inliers in sparse cases and in dense cases, it would degrade to the original selection strategy.

\subsection{Rigid Transformation Estimation} 
With the voted \textit{hypothetical inliers} $\{{}^{(\kappa)}\tilde{\mathcal{M}}_{\mathbf{c}_i}\mid \mathbf{c}_i \in \mathcal{C}_{seed}\}$ above, we estimate the optimal transformation parameter based on the \textit{Procrustes method}~\cite{gower1975generalized} to minimize the least-squares errors for each \textit{hypothetical inlier} group:
\begin{equation}\small
		\setlength{\abovedisplayskip}{2pt}
	\setlength{\belowdisplayskip}{2pt}
	\begin{split}
		\hat{\mathbf{R}}_i, \hat{\mathbf{t}}_i = \underset{\mathbf{R}, \mathbf{t}}{\arg\min} \sum_{\mathbf{c}_j \in {}^{(\kappa)}\tilde{\mathcal{M}}_{\mathbf{c}_i}} \omega_j \cdot\|\mathbf{R}^\top\mathbf{x}_j + \mathbf{t} - \mathbf{y}_j\|_2,
	\end{split}
\end{equation} 
where $\omega_j$ is the error weight computed by the neural spectral matching as in~\cite{bai2021pointdsc}.  
Then, we select the transformation parameter that maximizes the number of overlapped correspondences as the final optimal transformation estimation:
\begin{equation}\small
		\setlength{\abovedisplayskip}{2pt}
	\setlength{\belowdisplayskip}{0pt}
	\begin{split}
		\hat{\mathbf{R}}^*, \hat{\mathbf{t}}^* = \underset{\{\hat{\mathbf{R}}_i, \hat{\mathbf{t}}_i\}_{i=1}^{|\mathcal{C}|} }{\arg\max} \sum_{\mathbf{c}_j \in \mathcal{C}} v_i \cdot\mathbb{I}\left\{ \|\hat{\mathbf{R}}_i^\top\mathbf{x}_j + \hat{\mathbf{t}}_i - \mathbf{y}_j\|_2<\varepsilon \right\},
	\end{split}
\end{equation}
where $v_i=1-\|\hat{\mathbf{R}}_i^\top\mathbf{x}_j + \hat{\mathbf{t}}_i - \mathbf{y}_j\|_2^2 / \varepsilon^2$ is used to re-weight the inlier count as performed in \cite{jiang2021sampling}.
Finally, we refine it using all recovered inliers in a least-squares optimization as a common practice in~\cite{bai2021pointdsc,barath2018graph}.

\begin{table*}[t]
	\centering
	\resizebox{1.0\linewidth}{!}{
		\begin{tabular}{lccc|ccc|ccc|ccc|c}
			\toprule[1.5pt]
			& \multicolumn{3}{c|}{\textbf{3DMatch} ({FCGF})} & \multicolumn{3}{c|}{\textbf{3DMatch} ({FPFH})} & \multicolumn{3}{c|}{\textbf{KITTI} ({FCGF})} & \multicolumn{3}{c|}{\textbf{KITTI} ({FPFH})} & \\
			Models & RR($\uparrow$)  & RE($\downarrow$) & TE($\downarrow$) & RR($\uparrow$)  & RE($\downarrow$) & TE($\downarrow$) & RR($\uparrow$)  & RE($\downarrow$) & TE($\downarrow$) & RR($\uparrow$)  & RE($\downarrow$) & TE($\downarrow$) &   Sec.          \\ \midrule 
			FGR~\cite{zhou2016fast} & 79.17 & 2.93 & 8.56 & 41.10 & 4.05 & 10.09 & 96.58 & {0.38} & 22.30 & 1.26 & 1.69 & 47.18 & 1.39 \\
			SM~\cite{leordeanu2005spectral} & 86.57 & 2.29 & 7.07 & 55.82 & 2.94 & 8.13 &96.58 & 0.50 & \underline{19.88} & 75.50  & {0.66} & 15.01 & 0.02 \\
			RANSAC~\cite{fischler1981random}  & 91.50 & 2.49 & 7.54 & 73.57 & 3.55 & 10.04 & 97.66 & {\textbf{0.28}} & 22.61 & 89.37 & 1.22 & 25.88 & 6.43  \\
			TEASER++~\cite{yang2020teaser} & 85.77 & {2.73} & {8.66} & 75.48 & {2.48} & 7.31 & 83.24 & 0.84 & \textbf{12.48} & 64.14 & {1.04} & 14.85 & 0.07 \\
			DGR~\cite{choy2020deep} & 91.30 & 2.40 & 7.48 & 69.13 & 3.78 & 10.80 & 95.14 & {0.43} & 23.28 & 73.69 & 1.67 & 34.74 & 1.36 \\ 
			DHVR~\cite{lee2021deep} & 89.40 & 2.19 & 6.95 & 67.10 & 2.56 & 7.67 & -- & --  & --  & --  & --  & -- & 0.40 \\
   			SC2\_PCR~\cite{chen2022sc2}   & {\underline{93.10}} & \textbf{2.04} & \underline{6.53} & {\textbf{83.92}} & \underline{2.09} &\underline{6.66} & \underline{97.48} & 0.33 & {{20.66}} &{97.84} & {\underline{0.39}} & 9.09 & 0.09 \\
			PointDSC~\cite{bai2021pointdsc}   & 92.42 & \underline{{2.05}}  & {\textbf{6.49}} & 77.51 & {\textbf{2.08}}  &{\textbf{6.51}} & 97.66 & 0.47 & {{20.88}} & \underline{98.20} & {{0.58}} & \underline{7.27} & 0.11 \\
			\midrule
			\textbf{VBReg} & {\textbf{93.53}} & \textbf{2.04} & \textbf{{6.49}} &  {\underline{82.75}} & {{2.14}} & {{6.77}}& {\textbf{98.02}} & \underline{0.32} & {20.91} & {\textbf{98.92}} & \textbf{0.32} & {\textbf{7.17}} & 0.22 \\
			\bottomrule[1.5pt]
	\end{tabular}}
	\vspace{-2mm}
	\caption{Quantitative comparison on {3DMatch}~\cite{zeng20173dmatch} and {KITTI}~\cite{geiger2012we} benchmark datasets with descriptors {FCGF} and {FPFH}. The registration speed is achieved by computing  the averaged time cost on {3DMatch} with FCGF descriptor.}\label{tab:modelnet40}
	\vspace{-2mm}
	\label{3dmatch2}
\end{table*}

\section{Experiments}
\subsection{Experimental Settings}
\noindent\textbf{Implementation Details.}\label{sec:detail} 
For our variational non-local module, the number of iterations $L$ is 12, and the dimensions of the correspondence feature, random feature, and hidden feature are set to 128, 128, and 256, respectively. 
For our voting-based inlier sampling module, the size of \textit{hypothetical inliers} $\kappa$ is 40. 
For seed selection, seed ratio $v$ is $0.1$  and the lower bar of seed number $n$ is 1000. 
Our model is trained with $50$ epochs using Adam optimizer with learning rate $10^{-4}$ and weight decay $10^{-6}$.  
We utilize PyTorch to implement our project and perform all experiments on the server equipped with an Intel i5 2.2 GHz CPU and one Tesla V100 GPU. 
For simplicity, we name our \textbf{V}ariational \textbf{B}ayesian-based \textbf{Reg}istration framework as   \textbf{VBReg}. 

\noindent\textbf{Evaluation Metric.}
We use three metrics to evaluate our method, including (1) {Registration Recall (\textit{RR})}, the percent of the successful registration satisfying the error thresholds of  rotation and translation at the same time, (2) {average  Rotation Error (\textit{RE})} and (3) {average  Translation Error (\textit{TE})}:
\begin{equation}\small
		\setlength{\abovedisplayskip}{1pt}
	\setlength{\belowdisplayskip}{1pt}
	\begin{split}
	\operatorname{RE}(\hat{\mathbf{R}})=\arccos \frac{\operatorname{Tr}\left(\hat{\mathbf{R}}^\top \mathbf{R}^{*}\right)-1}{2},	\mathrm{TE}(\hat{\mathbf{t}})=\left\|\hat{\mathbf{t}}-\mathbf{t}^{*}\right\|_{2}^{2},
	\vspace{-2mm}
\end{split}
\end{equation}
where $\hat{\mathbf{R}}$ and $\hat{\mathbf{t}}$ are the predicted rotation matrix and rotation vector, respectively, while ${\mathbf{R}}^*$ and ${\mathbf{t}}^*$ are the corresponding ground truth. The average \textit{RE} and \textit{TE} are computed only on successful aligned point cloud pairs.

\begin{table}[]
	\centering
	\resizebox{1.0\linewidth}{!}{
		\begin{tabular}{llccccc}
			\toprule[1.5pt]
			Feature & Model &  5000  & 2500 & 1000 & 500 & 250\\ \midrule 
			\multirow{9}{*}{{FCGF}} 
			&FGR~\cite{zhou2016fast} & 18.6 & 19.4 & 16.9 & 16.0 & 12.4 \\
			&SM~\cite{leordeanu2005spectral} & 32.4 & 31.3 & 31.4 & 28.0 & 23.5   \\ 	
			&RANSAC~\cite{fischler1981random}  & 37.6 & 37.2 & 35.9 & 32.1 & 25.9 \\
			& TEASER++~\cite{yang2020teaser} & 42.8 & 42.4 & 39.5 & 34.5 & 25.7 \\
			& DHVR~\cite{lee2021deep} & 50.4 & 49.6 & 46.4 & 41.0 & \underline{34.6} \\
			&SC2\_PCR~\cite{chen2022sc2}  & {\underline{57.4}} & {\underline{56.5}} & {\underline{51.8}} & {\underline{46.4}} & {\textbf{36.2}}  \\
			&TR\_DE~\cite{chen2022deterministic} & {49.5} & {50.4} & {{48.4}} & {{43.4}} & {{34.3}} \\
			&PointDSC~\cite{bai2021pointdsc}   & {55.8} & {52.6}  & 46.8 & 37.7 & 26.7 \\ 
			\cmidrule{2-7}
			& \textbf{VBReg} & \textbf{58.3} & \textbf{56.8} & \textbf{52.9} & \textbf{47.2} & {34.5} \\
			\cmidrule{1-7}
			\multirow{9}{*}{{Predator}} 
			&FGR~\cite{zhou2016fast} & 36.4 & 38.2 & 39.7 & 39.6 & 38.0 \\
			&SM~\cite{leordeanu2005spectral} & 53.8 & 55.1 & 55.4 & 54.5 & 50.2 \\ 		
			&RANSAC~\cite{fischler1981random} & 62.3 & 62.8 & 62.4 & 61.5 & 58.2  \\ 
			&TEASER++~\cite{yang2020teaser} & 62.9 & 62.6 & 61.9 & 59.0 & 56.7\\
			& DHVR~\cite{lee2021deep} & 67.2 & {67.3} & 66.1 & 64.6 & 60.5 \\
			&SC2\_PCR~\cite{chen2022sc2} & \underline{69.5} & \underline{69.5} & \underline{{68.6}} & \underline{{65.2}} & \underline{{62.0}} \\
			&TR\_DE~\cite{chen2022deterministic} & {64.0} & {64.8} & {{61.7}} & {{58.8}} & {{56.5}} \\
			&PointDSC~\cite{bai2021pointdsc} & 68.1 & {67.3}  & 66.5 & 63.4 & 60.5 \\
			\cmidrule{2-7}
			&\textbf{VBReg} & \textbf{69.9} & \textbf{69.8}  & \textbf{68.7}  & \textbf{66.4} & \textbf{63.0}  \\
			\bottomrule[1.5pt]
	\end{tabular}}
	\vspace{-2mm}
	\caption{Registration recall (\textit{RR}) with different numbers of points on {3DLoMatch} benchmark dataset~\cite{huang2021predator}.}	\label{table:loindoor}
\end{table}

\subsection{Comparison with Existing Methods}
\noindent\textbf{Evaluation on 3DMatch.} We first evaluate our method on {3DMatch} benchmark~\cite{zeng20173dmatch}, which contains 46 training scenes, 8 validation scenes, and 8 test scenes. We first voxelize and down-sample the point cloud with 5cm voxel size and then leverage FCGF~\cite{choy2019fully} and FPFH~\cite{rusu2009fast} descriptors to construct the putative correspondences based on the feature nearest neighbor. 
We compare our method with eight state-of-the-art (SOTA) correspondence-based methods, where FGR~\cite{zhou2016fast}, SM~\cite{leordeanu2005spectral}, RANSAC (50k)~\cite{fischler1981random}, TEASER++~\cite{yang2020teaser}, and SC2\_PCR~\cite{chen2022sc2} are representative traditional methods, while DGR~\cite{choy2020deep}, DHVR~\cite{lee2021deep}, and PointDSC~\cite{bai2021pointdsc} are advanced deep learning-based methods. 
As shown in Table~\ref{3dmatch2}, in the FCGF setting, our method achieves the best performance in \textit{RR} and \textit{RE}  criteria while the same \textit{TE} with PointDSC. 
We need to highlight that \textit{RR} is a more important criterion than \textit{RE} and \textit{TE} since the rotation and translation errors are just calculated in a successful registration. 
It means that the higher \textit{RR} may include more challenging but successful registration cases, potentially increasing their errors.   
In the FPFH setting, it can be observed that our method can still achieve the best \textit{RR} score among all deep methods, but perform slightly worse than SC2\_PCR. 
Notably, compared to PointDSC (our baseline), the precision gain on \textit{RR} is impressive (5.24\%$\uparrow$), which benefits from the effectiveness of our variational non-local feature learning and the voting-based inlier searching. 

 \begin{figure}[t]
	\centering
	\includegraphics[width=\columnwidth]{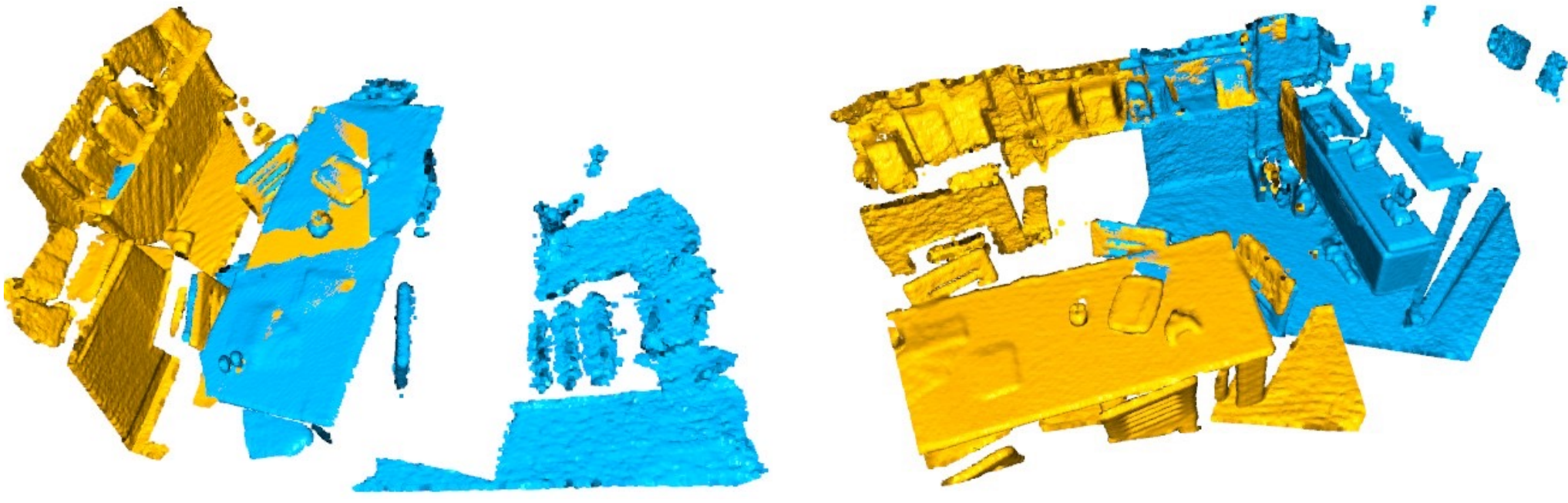}
	\vspace{-4mm}
	\caption{Registration visualization on {3DLoMatch}~\cite{geiger2012we}.} \label{fig:3dlomatch}
	\vspace{-6mm}
\end{figure}
 
\noindent\textbf{Evaluation on 3DLoMatch.}
We further test our method on {3DLoMatch} benchmark~\cite{huang2021predator}. 
Compared to {3DMatch} sharing more than $30\%$ overlap, the overlaps of point cloud pairs in {3DLoMatch} just lie in $10\%\sim30\%$, thus presenting much more challenges. 
We leverage FCGF~\cite{choy2019fully} and recently popular Predator~\cite{huang2021predator} as the feature descriptors for putative correspondence generation. 
We choose six traditional methods: FGR~\cite{zhou2016fast}, SM~\cite{leordeanu2005spectral}, RANSAC (50k)~\cite{fischler1981random}, TEASER++~\cite{yang2020teaser}, SC2\_PCR~\cite{chen2022sc2}, and TR\_DE~\cite{chen2022deterministic}, and two deep methods: DHVR~\cite{lee2021deep}, and PointDSC~\cite{bai2021pointdsc} for comparison. 
The registration recalls (\textit{RR}) under different numbers of correspondences are listed in Table~\ref{table:loindoor}. 
It can be observed that regardless of FCGF or Predator descriptor, our method almost achieves the best performance on all settings, except for FCGF setting with 250 points. 
Notably, in more challenging cases, the performance advantage over PointDSC is further expanded (+9.5\% and +7.8\% in the cases of FCGF with 500 and 250 points). 
These experimental results further demonstrate the outstanding robustness of our method when encountering those extremely low-overlapping cases (high outlier ratios). 
Some visualization results are listed in Fig.~\ref{fig:3dlomatch} and the \textit{RR} changes under different inlier ratios are presented in Fig.~\ref{ablation_plots1} (c), which suggests that our performance gains are mainly brought by model robustness in extremely high outlier situations.

\begin{table*}[]
	\centering
	\resizebox{1.0\textwidth}{!}{
		\begin{tabular}{lcc|ccccc|ccccc|c}
			\toprule[1.5pt]
			&\multicolumn{2}{c|}{\textbf{3DMatch}}& \multicolumn{5}{c|}{\textbf{3DLoMatch} (FCGF)} & \multicolumn{5}{c|}{\textbf{3DLoMatch} (Predator)}  \\ 
			{Model} & {FCGF} & {FPFH} & 5000 & 2500 & 1000 & 500 & 250 & 5000 & 2500 & 1000 & 500 & 250 & Sec. \\
			\midrule 
			PointDSC \textit{w/} SCNonlocal$^{xyz}$ & {92.42} & {77.51} & {55.8} & 52.6 & 46.8 & 37.7 & 26.7 & 68.1 & 67.3  & 66.5 & 63.4 & 60.5 & 0.11 \\ 
			PointDSC \textit{w/} VBNonlocal$^{xyz}$ & 93.04 & {80.16} & \underline{57.7} & 55.6 & 50.2 & 39.9 & 26.1 & \underline{69.7} & \underline{69.6}  & 67.9 & 64.9 & \underline{61.9} & 0.17 \\ 
			 \cdashline{1-14}[1pt/2pt]
			 PointDSC \textit{w/} SCNonlocal$^{feat}$ & 92.36 & 77.76 & 54.6 & 50.6 & 44.9 & 36.8 & 25.4 & 69.2 & 68.6  & 67.9 & 63.5 & 59.9 & 0.13 \\ 
            PointDSC \textit{w/} VBNonlocal$^{feat}$ &  {93.04} & {80.53}& {56.9} & \underline{56.5} & 50.9 & 42.2 & 28.9 & 69.2 &  68.7 & 68.0 & 64.6 & 60.6 & 0.18 \\  
			\cdashline{1-14}[1pt/2pt]
			PointDSC \textit{w/} SCNonlocal$^{cls}$ & {92.98} & {78.99} & {54.1} & 52.2 & 46.0 & 38.7 & {27.7} & 67.6 & 66.9 & 67.2 & 63.7 & {60.2}& 0.11\\  
			PointDSC \textit{w/} VBNonlocal$^{feat}$+Vote & \underline{93.41} & \underline{81.21} & \textbf{58.3} & \underline{56.5} & \underline{51.9} & \underline{44.7} & \underline{31.1} & {69.3} & {69.5} & \underline{68.2} & \underline{65.3} & {61.2} & 0.20\\  
			PointDSC \textit{w/} VBNonlocal$^{feat}$+Vote+CS& \textbf{93.53} & \textbf{82.75} &\textbf{58.3} & \textbf{56.8} & \textbf{52.9} & \textbf{47.2} & \textbf{34.5} & \textbf{69.9} & \textbf{69.8}  & \textbf{68.7}  & \textbf{66.4} & \textbf{63.0} & 0.22 \\ 
			\midrule
			Iteration times $L = 6$ & \underline{93.41} & \underline{82.32} & {58.1} & \textbf{57.1} & \underline{52.9} & \textbf{48.3} & \textbf{34.8} & {69.7} & \underline{69.7} & \textbf{68.7} & {66.3} & \textbf{63.8} & 0.19 \\  
			Iteration times $L = 9$ & \underline{93.41} & {81.58} & \underline{58.2} & \underline{57.0} & \textbf{53.2} & \underline{47.6} & 32.5 & \textbf{70.0} & {69.4}  & \underline{68.5} & \textbf{66.9} & \underline{63.2} & 0.20\\  
			Iteration times $L = 12$* & \textbf{93.53} & \textbf{82.75} & \textbf{58.3} & {56.8} & \underline{52.9} & {47.2} & \underline{34.5} &  \underline{69.9} & \textbf{69.8}  & \textbf{68.7}  & \underline{66.4} & {63.0}& 0.22\\  
			\midrule
			Random feat. dim. ${\tilde{d}}=32$ & \underline{93.41} & \underline{82.38} & \underline{58.0} & \textbf{57.5} & \textbf{53.5} & \underline{48.1} & \underline{34.8} & \underline{69.7} & \textbf{70.0}  & \underline{68.6} & {66.3} & \textbf{63.3} & 0.20 \\  
			Random feat. dim. ${\tilde{d}}=64$ & \underline{93.41} & {81.45} & {57.9} & \underline{56.9} & {52.8} &\textbf{48.6} & \textbf{35.0} &  {69.6} & {69.7} & {68.4} & \textbf{66.5} & \underline{62.3} & 0.21  \\  
			Random feat. dim. ${\tilde{d}}=128$* & \textbf{93.53} & \textbf{82.75} &\textbf{58.3} & {56.8} & \underline{52.9} & {47.2} & {34.5} & \textbf{69.9} & \underline{69.8}  & \textbf{69.3}  & \underline{66.4} & \textbf{63.3} & 0.22 \\ 
			\bottomrule[1.5pt]
	\end{tabular}}
	\vspace{-2mm}
	\caption{Ablation studies on {3DMatch}~\cite{zeng20173dmatch} and {3DLoMatch}~\cite{huang2021predator} datasets. \textit{SCNonlocal}: Spatial consistency-guided non-local network; \textit{VBNonlocal}: Variational Bayesian-based non-local network;  \textit{Vote}: Voting-based inlier searching; \textit{CS}: Conservative seed selection.}
	\label{different2}
\end{table*}

\noindent\textbf{Evaluation on KITTI.} Finally, we evaluate our method on the outdoor LIDAR-scanned driving scenarios from {KITTI} dataset~\cite{geiger2012we}. 
In line with \cite{choy2019fully}, we utilize sequences 0-5, 6-7, and 8-10 as the training set, validation set, and test set, respectively. 
Also, as the setting in  \cite{bai2020d3feat,choy2019fully}, we further refine the ground-truth rigid transformations using ICP~\cite{besl1992method}  and only collect the point cloud pairs far away from each other at most 10m as the test dataset. 
We downsample the point cloud with a voxel size of 30cm and exploit FCGF~\cite{choy2019fully} and FPFH~\cite{rusu2009fast} descriptors for correspondence construction, respectively. 
The compared methods are consistent with those in \textit{3DMatch}. 
The comparison results are listed in Table \ref{3dmatch2}. 
For the FCGF setting, our method can achieve the best scores on the most important \textit{RR} criterion while for the FPFH setting, our method can consistently achieve the best scores on all criteria. 

\begin{figure}[htb]
	\centering
	\subfloat[ The distribution of inlier feature similarity on {3DMatch}~\cite{zeng20173dmatch}.]{\includegraphics[width = .49\linewidth]{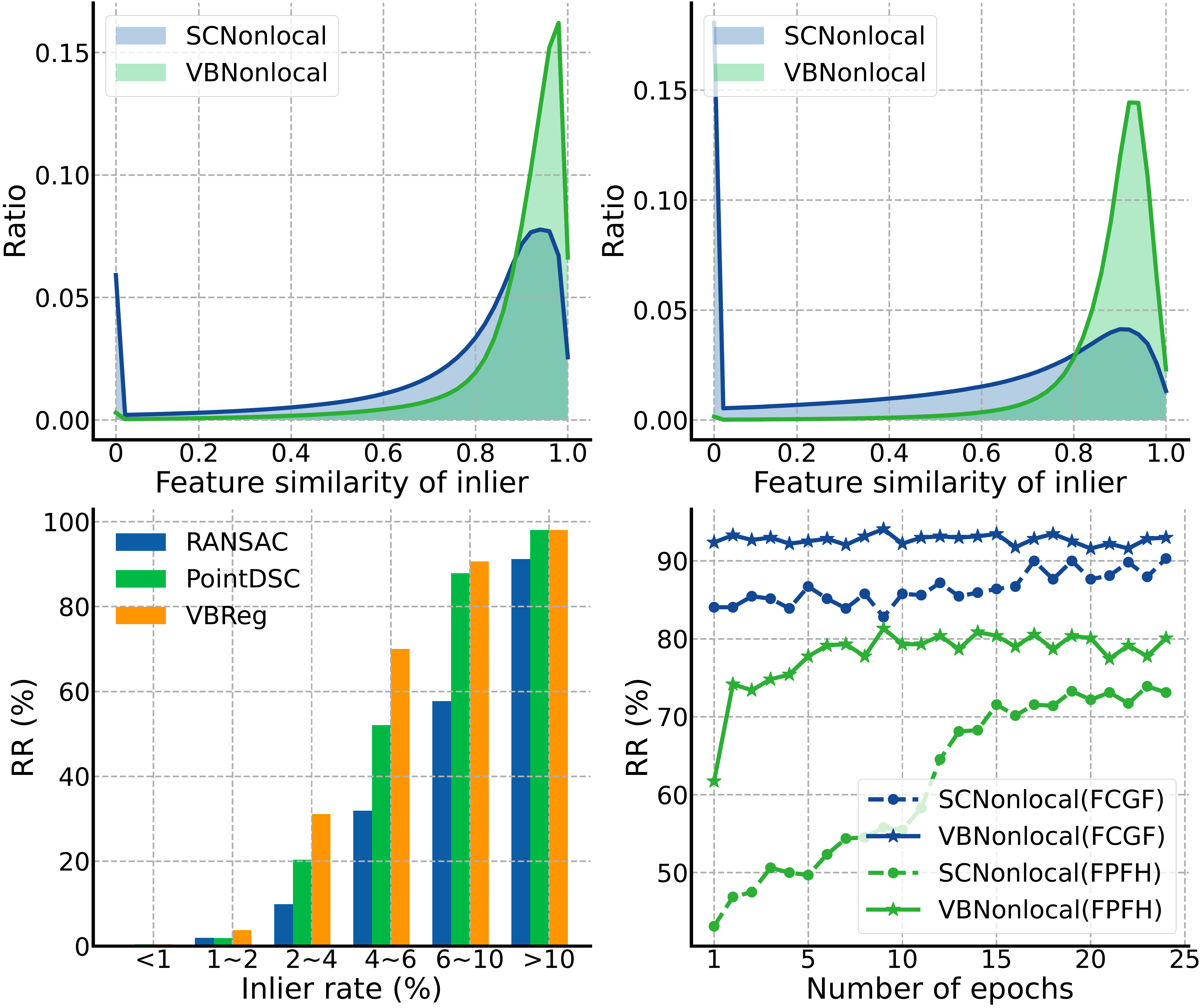}}\hspace{0.2mm}
	\subfloat[The distribution of inlier feature similarity on {3DLoMatch}~\cite{huang2021predator}.]{\includegraphics[width = .49\linewidth]{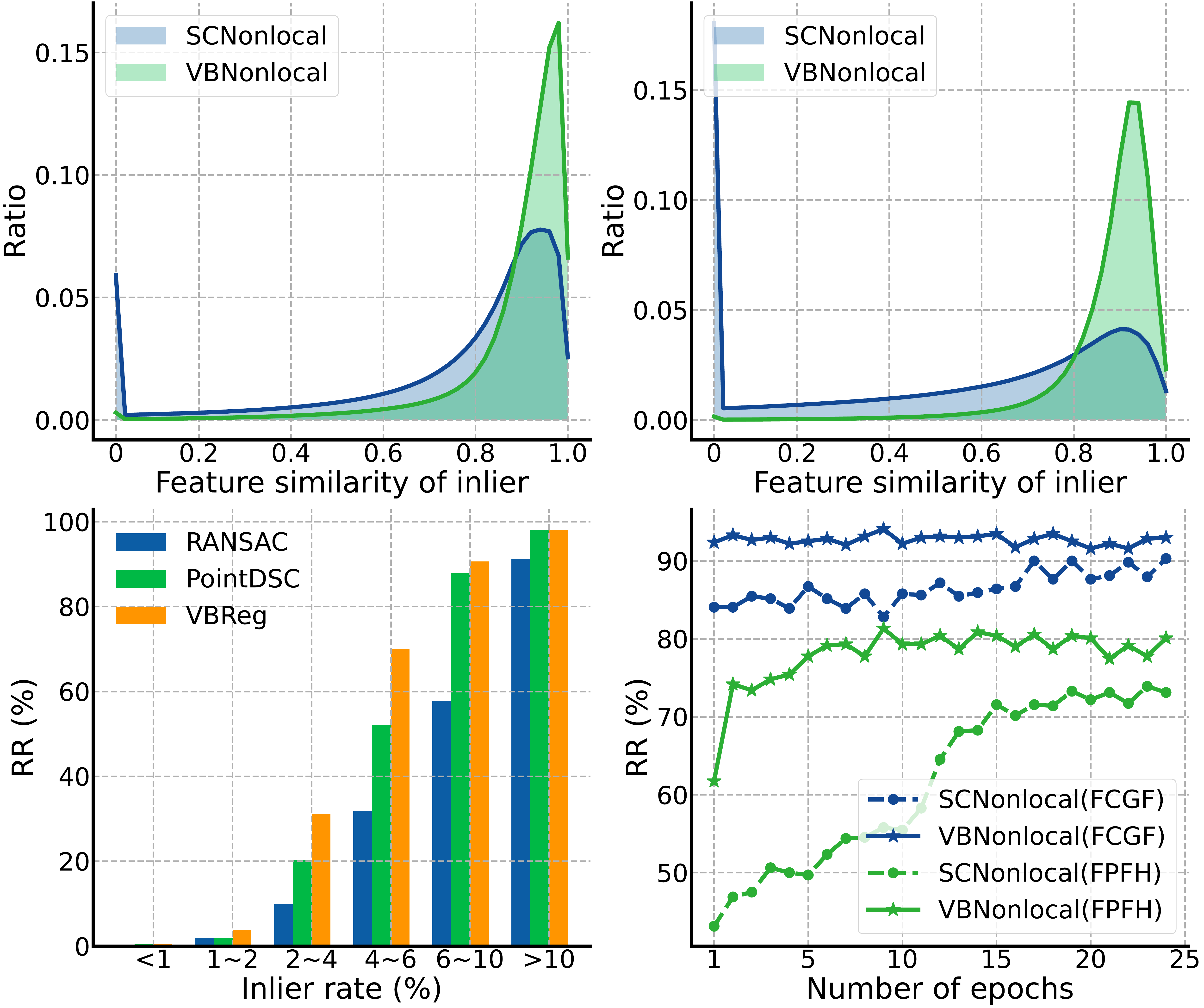}}\\
	\subfloat[\textit{RR} under different inlier ratios.]{\includegraphics[width = .49\linewidth]{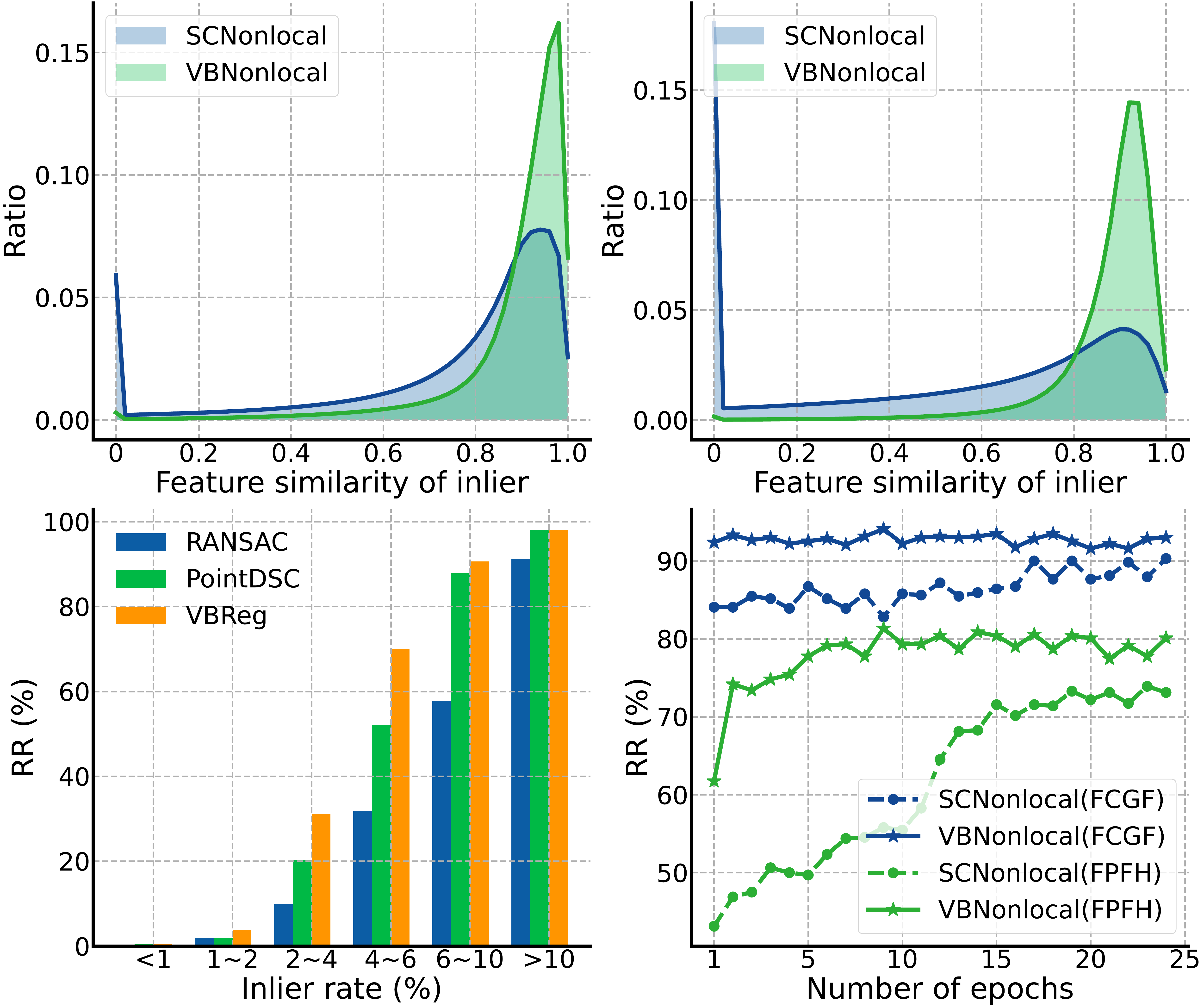}}\hspace{0.2mm}
	\subfloat[Training curves.]{\includegraphics[width = .49\linewidth]{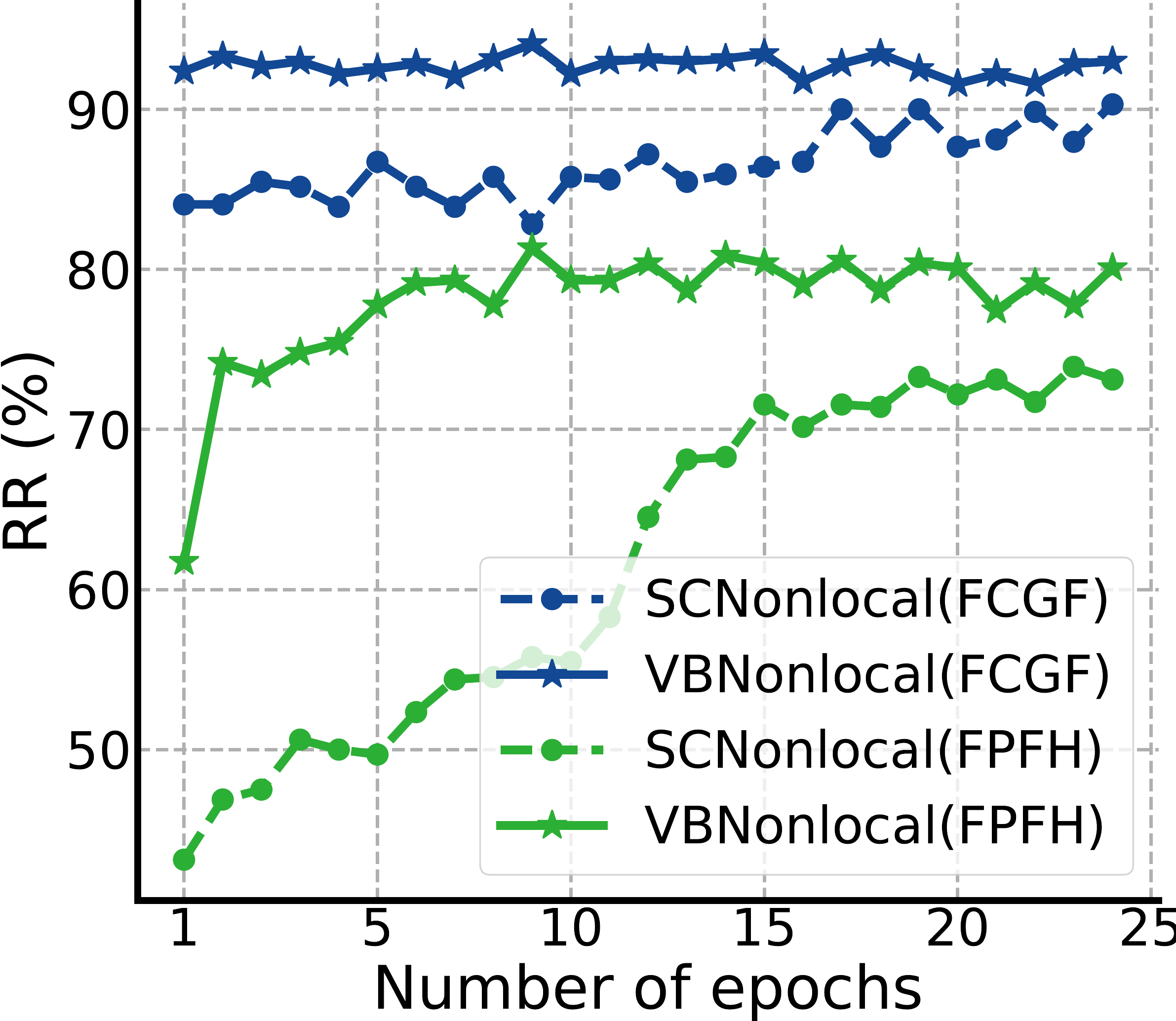}}
	\vspace{-2mm}
	\caption{(a) and (b): The distribution of the learned feature similarity of inliers; (c):  \textit{RR} under different inlier ratios; (d): \textit{RR} on the validation sets of {3DMatch}~\cite{zeng20173dmatch} (FCGF) and {3DMatch} (FPFH) under different training epochs.}
	\label{ablation_plots1}
\end{figure}

\subsection{Ablation Studies and Analysis}  
%
%
\noindent\textbf{Variational Non-local Network.} 
We first take PointDSC as our baseline and compare our proposed variational Bayesian non-local network (\textit{VBNonlocal}) to the spatial consistency-based non-local network (\textit{SCNonlocal}) to highlight the effectiveness of our proposed method.
\textbf{(1)}: We first compare their performance difference under two types of network input: \textit{VBNonlocal}$^{xyz}$ and \textit{SCNonlocal}$^{xyz}$ indicate using concatenated correspondence coordinate as input while \textit{VBNonlocal}$^{feat}$ and \textit{SCNonlocal}$^{feat}$ represent using concatenated 
coordinate and descriptor of correspondence as input. 
 As shown in the top block in Table~\ref{different2}, with each data type as input, our {\textit{VBNonlocal}} can consistently achieve significant performance gains. 
 Especially, on {3DMatch} with FPFH descriptor, {\textit{VBNonlocal}$^{feat}$} brings $2.77\%$ \textit{RR} improvement and on {3DLoMatch} with 500, 1000 and 2500 points, the \textit{RR} improvements even can reach $5.4\%$, $6\%$ and $5.9\%$, respectively. 
 These impressive results support our view that our Bayesian-driven long-range dependency modeling can effectively learn the discriminative inlier/outlier features for reliable inlier search. 
\textbf{(2)}: Then, to further highlight the superiority of our variational inference-guided feature learning, we also try to add classification loss on the features produced by each iteration in \textit{SCNonlocal} to guide their learning (denoted as \textit{SCNonlocal}$^{cls}$). 
As shown in the fifth row in Table~\ref{different2}, such loss-based label-propagation way just can achieve very limited performance gain and even sometimes degrades score. It demonstrates the superiority of our label-dependent posterior guidance for (prior) feature learning. 
Also, owing to such posterior guidance,  the training curves in Fig.~\ref{ablation_plots1} (d) show that our method can achieve significantly faster convergence speed than \textit{SCNonlocal}.

\noindent\textbf{Discriminative Feature Learning?} 
In order to verify whether \textit{VBNonlocal} can learn more discriminative features than \textit{SCNonlocal}, we visualize the distribution of the feature similarity of inliers on {3DMatch} (Fig.~\ref{ablation_plots1} (a)) and {3DLoMatch} (Fig.~\ref{ablation_plots1} (b)). 
As we can see, our learned inlier features own much higher similarities (approximate to 1) than  \textit{SCNonlocal} on both datasets, which demonstrates that our proposed Bayesian-inspired non-local network truly can promote more discriminative correspondence-feature learning.

\noindent\textbf{Variational Non-local Setting.} 
We further test the performance changes under different \textit{VBNonlocal} settings. 
\textbf{(1)} We first test the model robustness under different numbers of non-local iterations. 
The second block in Table~\ref{different2} verifies that our method is robust to the iteration time and tends to consistently achieve outstanding \textit{RR} score. 
\textbf{(2)} Then, the bottom block in Table~\ref{different2} further shows our model stability to different dimensions of random features.  


 \noindent\textbf{Voting-based Inlier Searching.} 
Furthermore, we evaluate the performance contribution of the proposed voting-based inlier searching strategy (\textit{Vote}). 
As we can see in the sixth row of Table~\ref{different2}, voting strategy can consistently achieve performance improvement regardless in {3DMatch} or in more challenging {3DLoMatch}. 
It mainly benefits from the high-quality \textit{hypothetical inliers} sampled by our voting policy. 

 \noindent\textbf{Conservative Seed Selection.} Finally, we test the effectiveness of the proposed conservative seed selection strategy (\textit{CS}) motivated by Theorem~\ref{theorem1}. 
As we can see in the seventh row of Table~\ref{different2}, \textit{CS} can achieve consistent performance gain in each setting. 
Especially, in the cases of fewer points (\eg, FCGF setting with 250 and 500 points), the improvement is much more significant (+3.4\%, +2.5\%). 
As the analysis before, in sparse point clouds, the original inlier selection strategy like in~\cite{bai2021pointdsc} is aggressive and prone to miss too many inlier seeds. Instead, our conservative selection strategy can effectively mitigate it as well as keep registration efficiency.   

\section{Conclusion} 
In this paper, we adapted the variational Bayesian inference into the non-local network and developed the effective Bayesian-guided long-term dependencies for discriminative correspondence-feature learning. 
To achieve effective variational inference, a probabilistic graphical model was customized over our non-local network, and the variational low bound was derived as the optimization objective for model training. 
In addition, we proposed a Wilson score-based voting mechanism for high-quality inlier sampling and theoretically verified its superiority over RANSAC. 
Extensive experiments on indoor/outdoor datasets demonstrated its promising performance. 

\section{Acknowledgments} 
This work was supported by the National Science Fund of China (Grant Nos. U1713208, U62276144).

{\small
\bibliographystyle{ieee_fullname}
\bibliography{PaperForReview}
}

\onecolumn 

\section*{\centering \LARGE Supplementary Material for ``Robust Outlier Rejection for 3D Registration with Variational Bayes"} 
\section{Proof of Variational Lower Bound}
\label{proof}
We propose two forms of variational lower bounds on our variational non-local network, including the point cloud-wise lower bound (\S\ref{bound1}) and the point-wise lower bound (\S\ref{bound2}).
The former demonstrates the point-cloud distribution, while the latter shows a more detailed distribution for each point. For clarity, we apply the former in our paper.

\begin{figure}[ht]
	\centering
	\includegraphics[width=0.5\columnwidth]{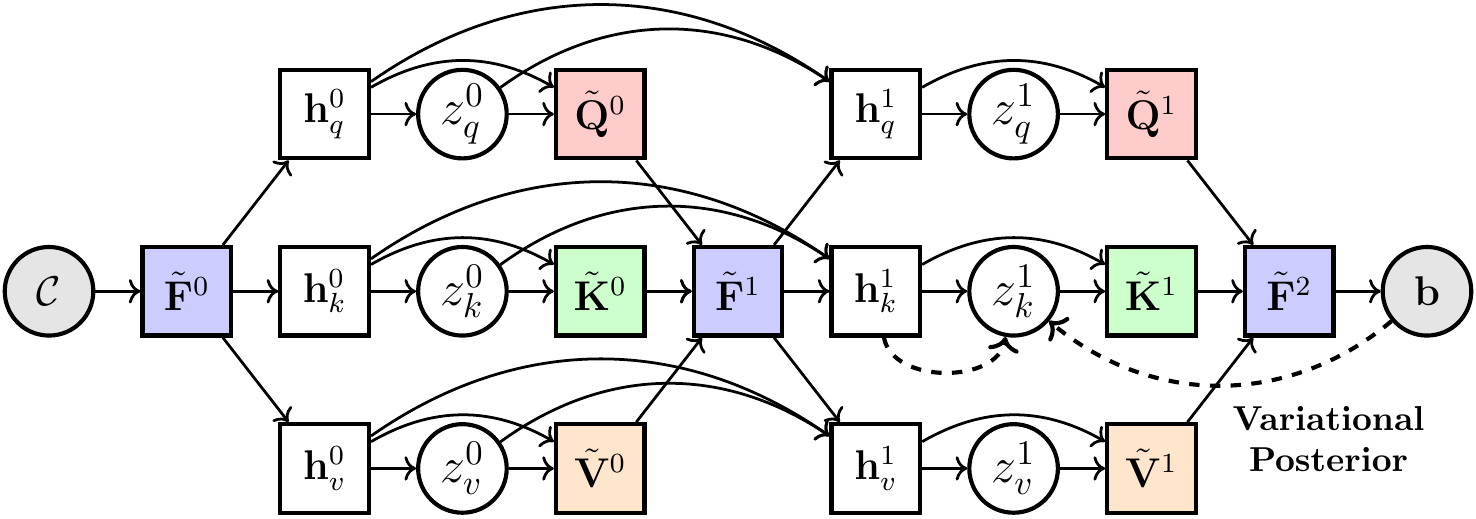}
	\caption{Probabilistic graphical model for our variational non-local network.  For simplicity, we just demonstrate two iterations. The white circles indicate the random features and the white squares denote the deterministic hidden features. 
		The solid line represents the inlier/outlier prediction process (generative process) and the dashed line denotes the label-dependent posterior encoder (inference model). 
		We just show the inference model for ${z}^1_k$.} \label{fig:model}
	\label{graph}
	\vspace{-4mm}
\end{figure}

\subsection{Point Cloud-wise Variational Lower Bound} \label{bound1}
Given the putative correspondences $\mathcal{C}=\{\mathbf{c}_1, \mathbf{c}_2, ..., \mathbf{c}_N\}$ and their inlier/outlier labels $\mathbf{b}=\{b_1, b_2, ..., b_N \mid b_i \in \{0, 1\}\}$, we first inject a set of random feature variables $z^{<L}_{q,k,v}=\{z^{l, i}_{q}, z^{l, i}_{k}, z^{l, i}_{v}\}_{0\leq l <L, 1\leq i\leq N}$ into its log-likelihood correspondence labels  $p_\theta(\mathbf{b}\mid \mathcal{C})$ and the initial variational lower bound can be derived using the Jensen’s inequality as follows: 
\begin{equation}\label{init1}
	\begin{split}
		&\ln p_\theta(\mathbf{b}\mid \mathcal{C}) = \ln \int_{z^{<L}_{q,k,v}} p_\theta(\mathbf{b}, z^{<L}_{q,k,v}\mid \mathcal{C}) \\ 
		= & \ln \int_{z^{<L}_{q,k,v}} q_\phi(z^{<L}_{q,k,v}\mid\mathcal{C}, \mathbf{b}) \frac{p_\theta(\mathbf{b}, z^{<L}_{q,k,v}\mid \mathcal{C})}{q_\phi(z^{<L}_{q,k,v}\mid  \mathcal{C},\mathbf{b})} 
		= \ln \mathbb{E}_{q_\phi(z^{<L}_{q,k,v}\mid\mathcal{C}, \mathbf{b})}\left[\frac{p_\theta(\mathbf{b}, z^{<L}_{q,k,v}\mid \mathcal{C})}{q_\phi(z^{<L}_{q,k,v}\mid  \mathcal{C},\mathbf{b})}\right] \\
		\geq & \mathbb{E}_{q_\phi(z^{<L}_{q,k,v}\mid\mathcal{C}, \mathbf{b})}\left[\ln\frac{p_\theta(\mathbf{b}, z^{<L}_{q,k,v}\mid \mathcal{C})}{q_\phi(z^{<L}_{q,k,v}\mid  \mathcal{C},\mathbf{b})}\right] 
		\overset{(1)}{=} \mathbb{E}_{q_\phi(z^{<L}_{q,k,v}\mid\mathcal{C}, \mathbf{b})}\left[\ln\frac{p_\theta(\mathbf{b}\mid z^{<L}_{q,k,v}, \mathcal{C})\cdot p_\theta(z^{<L}_{q,k,v}\mid \mathcal{C})}{q_\phi(z^{<L}_{q,k,v}\mid  \mathcal{C},\mathbf{b})}\right],
	\end{split}
\end{equation}
where step (1) is based on the chain rule in probability theory. 
Next, based on the defined conditional dependencies of random variables in our probabilistic graphical model (Fig.~\ref{graph}), the chain rule is also used to factorize  the posterior distribution $q_\phi(z^{<L}_{q,k,v}\mid\mathcal{C}, \mathbf{b})$ and the prior distribution $p_\theta(z^{<L}_{q,k,v}\mid\mathcal{C})$ in Eq.~\ref{init1}:
\begin{equation} \label{post_1}
	\begin{split}
		q_\phi(z^{<L}_{q,k,v}\mid\mathcal{C}, \mathbf{b}) = & q_\phi(z^{L-1}_{q,k,v}\mid z^{<L-1}_{q,k,v}, \mathcal{C}, \mathbf{b}) \cdot q_\phi(z^{<L-1}_{q,k,v}\mid \mathcal{C}, \mathbf{b}) \\
		= & q_\phi(z^{L-1}_{q,k,v}\mid z^{<L-1}_{q,k,v}, \mathcal{C}, \mathbf{b}) \cdot q_\phi(z^{L-2}_{q,k,v}\mid z^{<L-2}_{q,k,v}, \mathcal{C}, \mathbf{b}) \cdot q_\phi(z^{<L-2}_{q,k,v}\mid \mathcal{C}, \mathbf{b}) \\
		&...... \\
		=& \Pi^{L-1}_{l=0}{q_\phi(z^{l}_{q,k,v}\mid z^{<l}_{q,k,v}, \mathcal{C}}, \mathbf{b}) \\
	\end{split}
\end{equation}
\begin{equation} \label{prior_1}
	\begin{split}
		p_\theta(z^{<L}_{q,k,v}\mid\mathcal{C}) = & p_\theta(z^{L-1}_{q,k,v}\mid z^{<L-1}_{q,k,v}, \mathcal{C}) \cdot p_\theta(z^{<L-1}_{q,k,v}\mid \mathcal{C}) \\
		= & p_\theta(z^{L-1}_{q,k,v}\mid z^{<L-1}_{q,k,v}, \mathcal{C}) \cdot p_\theta(z^{L-2}_{q,k,v}\mid z^{<L-2}_{q,k,v}, \mathcal{C}) \cdot p_\theta(z^{<L-2}_{q,k,v}\mid \mathcal{C}) \\
		&...... \\
		=& \Pi^{L-1}_{l=0}{p_\theta(z^{l}_{q,k,v}\mid {z^{<l}_{q,k,v}, \mathcal{C}})}. 
	\end{split}
\end{equation}
By inserting the factorized posterior (Eq.~\ref{post_1}) and prior (Eq.~\ref{prior_1}) distributions into Eq.~\ref{init1}, we can drive the detailed variational lower bound as below:
\begin{equation}\label{detail}\small
	\begin{split}
		&\ln p_\theta(\mathbf{b}\mid \mathcal{C}) = \ln \int_{z^{<L}_{q,k,v}} p_\theta(\mathbf{b}, z^{<L}_{q,k,v}\mid \mathcal{C}) \\ 
	    \geq & \mathbb{E}_{q_\phi(z^{<L}_{q,k,v}\mid\mathcal{C}, \mathbf{b})}\left[\ln\frac{p_\theta(\mathbf{b}\mid z^{<L}_{q,k,v}, \mathcal{C})\cdot p_\theta(z^{<L}_{q,k,v}\mid \mathcal{C})}{q_\phi(z^{<L}_{q,k,v}\mid  \mathcal{C},\mathbf{b})}\right] \\
		 = & \mathbb{E}_{\Pi^{L-1}_{l=0}{q_\phi(z^{l}_{q,k,v}\mid z^{<l}_{q,k,v}, \mathcal{C}}, \mathbf{b})}\left[\ln\frac{p_\theta(\mathbf{b}\mid z^{<L}_{q,k,v}, \mathcal{C})\cdot \Pi^{L-1}_{l=0}{p_\theta(z^{l}_{q,k,v}\mid {z^{<l}_{q,k,v}, \mathcal{C}})}}{\Pi^{L-1}_{l=0}{q_\phi(z^{l}_{q,k,v}\mid z^{<l}_{q,k,v}, \mathcal{C}}, \mathbf{b})}\right] \\
		= & \mathbb{E}_{\Pi^{L-1}_{l=0}{q_\phi(z^{l}_{q,k,v}\mid z^{<l}_{q,k,v}, \mathcal{C}}, \mathbf{b})}\left[\ln p_\theta(\mathbf{b}\mid z^{<L}_{q,k,v}, \mathcal{C}) - \ln\frac{\Pi^{L-1}_{l=0}{q_\phi(z^{l}_{q,k,v}\mid z^{<l}_{q,k,v}, \mathcal{C}}, \mathbf{b})}{\Pi^{L-1}_{l=0}{p_\theta(z^{l}_{q,k,v}\mid {z^{<l}_{q,k,v}, \mathcal{C}})}}\right] \\
		= & \mathbb{E}_{\Pi^{L-1}_{l=0}{q_\phi(z^{l}_{q,k,v}\mid z^{<l}_{q,k,v}, \mathcal{C}}, \mathbf{b})}\left[\ln p_\theta(\mathbf{b}\mid z^{<L}_{q,k,v}, \mathcal{C}) - \sum^{L-1}_{l=0}\ln\frac{{q_\phi(z^{l}_{q,k,v}\mid z^{<l}_{q,k,v}, \mathcal{C}}, \mathbf{b})}{{p_\theta(z^{l}_{q,k,v}\mid {z^{<l}_{q,k,v}, \mathcal{C}})}}\right] \\
		= & \mathbb{E}_{\Pi^{L-1}_{l=0}{q_\phi(z^{l}_{q,k,v}\mid z^{<l}_{q,k,v}, \mathcal{C}}, \mathbf{b})}\left[\ln p_\theta(\mathbf{b}\mid z^{<L}_{q,k,v}, \mathcal{C})\right] - \mathbb{E}_{\Pi^{L-1}_{l=0}{q_\phi(z^{l}_{q,k,v}\mid z^{<l}_{q,k,v}, \mathcal{C}}, \mathbf{b})}\left[\sum^{L-1}_{l=0}\ln\frac{{q_\phi(z^{l}_{q,k,v}\mid z^{<l}_{q,k,v}, \mathcal{C}}, \mathbf{b})}{{p_\theta(z^{l}_{q,k,v}\mid {z^{<l}_{q,k,v}, \mathcal{C}})}}\right] \\
		= & \mathbb{E}_{\Pi^{L-1}_{l=0}{q_\phi(z^{l}_{q,k,v}\mid z^{<l}_{q,k,v}, \mathcal{C}}, \mathbf{b})}\left[\ln p_\theta(\mathbf{b}\mid z^{<L}_{q,k,v}, \mathcal{C})\right] - \sum^{L-1}_{l=0}\mathbb{E}_{\Pi^{L-1}_{l=0}{q_\phi(z^{l}_{q,k,v}\mid z^{<l}_{q,k,v}, \mathcal{C}}, \mathbf{b})}\left[\ln\frac{{q_\phi(z^{l}_{q,k,v}\mid z^{<l}_{q,k,v}, \mathcal{C}}, \mathbf{b})}{{p_\theta(z^{l}_{q,k,v}\mid {z^{<l}_{q,k,v}, \mathcal{C}})}}\right] \\
		= & \mathbb{E}_{\Pi^{L-1}_{l=0}{q_\phi(z^{l}_{q,k,v}\mid z^{<l}_{q,k,v}, \mathcal{C}}, \mathbf{b})}\left[\ln p_\theta(\mathbf{b}\mid z^{<L}_{q,k,v}, \mathcal{C})\right] - \sum^{L-1}_{l=0}\mathbb{E}_{{\Pi^{ l}_{\tau=0}q_\phi(z^{\tau}_{q,k,v}\mid z^{<\tau}_{q,k,v}, \mathcal{C}}, \mathbf{b})}\left[\ln\frac{{q_\phi(z^{l}_{q,k,v}\mid z^{<l}_{q,k,v}, \mathcal{C}}, \mathbf{b})}{{p_\theta(z^{l}_{q,k,v}\mid {z^{<l}_{q,k,v}, \mathcal{C}})}}\right] \\
		= & \mathbb{E}_{\Pi^{L-1}_{l=0}{q_\phi(z^{l}_{q,k,v}\mid z^{<l}_{q,k,v}, \mathcal{C}}, \mathbf{b})}\left[\ln p_\theta(\mathbf{b}\mid z^{<L}_{q,k,v}, \mathcal{C})\right] - \sum^{L-1}_{l=0}\mathbb{E}_{\Pi^{l-1}_{\tau=0}q_\phi(z^{\tau}_{q,k,v}\mid z^{<\tau}_{q,k,v}, \mathcal{C}, \mathbf{b})}\left[ \operatorname{D_{KL}}\left(q_\phi(z^{l}_{q,k,v}\mid z^{<l}_{q,k,v}, \mathcal{C}, \mathbf{b})||  p_\theta(z^{l}_{q,k,v}\mid {z^{<l}_{q,k,v}, \mathcal{C}})\right)\right] \\
		= & \operatorname{ELBO}(\theta, \phi)
	\end{split}
\end{equation}

In our implementation, we use the deterministic hidden features $\{\mathbf{h}^l_{q,k,v}\}^{L-1}_{l=0}$ to summarize the historical information in previous iterations (\ie, the condition parts of prior and posterior distributions) so that the lower bound can be rewritten as:
\begin{equation}\small
	\begin{split}
		& \mathbb{E}_{\Pi^{L-1}_{l=0}{q_\phi(z^{l}_{q,k,v}\mid \underbrace{z^{<l}_{q,k,v}, \mathcal{C}}_{\mathbf{h}^l_{q,k,v}}}, \mathbf{b})}\Bigg[\ln {p_\theta}(\mathbf{b}\mid \underbrace{z^{<L}_{q,k,v}, \mathcal{C}}_{\tilde{\mathbf{F}}^L})\Bigg] - \sum^{L-1}_{l=0}\mathbb{E}_{\Pi^{l-1}_{\tau=0}q_\phi(z^{\tau}_{q,k,v}\mid z^{<\tau}_{q,k,v}, \mathcal{C}, \mathbf{b})} \Bigg[\operatorname{D_{KL}}\Bigg(q_\phi(z^{l}_{q,k,v}\mid \underbrace{z^{<l}_{q,k,v}, \mathcal{C}}_{\mathbf{h}^l_{q,k,v}}, \mathbf{b})||  p_\theta(z^{l}_{q,k,v}\mid {\underbrace{z^{<l}_{q,k,v}, \mathcal{C}}_{\mathbf{h}^l_{q,k,v}}})\Bigg)\Bigg] \\
		=& \mathbb{E}_{\Pi^{L-1}_{l=0}{q_\phi(z^{l}_{q,k,v}\mid {\mathbf{h}^l_{q,k,v}}}, \mathbf{b})}\left[\ln y_\theta(\mathbf{b}\mid {\tilde{\mathbf{F}}^L})\right] - \sum^{L-1}_{l=0}\mathbb{E}_{\Pi^{l-1}_{\tau=0}q_\phi(z^{\tau}_{q,k,v}\mid z^{<\tau}_{q,k,v}, \mathcal{C}, \mathbf{b})} \left[\operatorname{D_{KL}}\left(q_\phi(z^{l}_{q,k,v}\mid {\mathbf{h}^l_{q,k,v}}, \mathbf{b})||  p_\theta(z^{l}_{q,k,v}\mid {{\mathbf{h}^l_{q,k,v}}})\right)\right] 
	\end{split}
\end{equation}
where we use the correspondence features $\tilde{\mathbf{F}}^L$  of the last non-local iteration to summarize the condition parts of $p_\theta(z^{l}_{q,k,v}\mid {z^{<l}_{q,k,v}, \mathcal{C}})$. 
Also, to avoid ambiguity, we denote the label prediction model $p_\theta(\mathbf{b}\mid \tilde{\mathbf{F}})$ as $y_\theta(\mathbf{b}\mid \tilde{\mathbf{F}})$. 


\subsection{Point-wise Variational Lower Bound} \label{bound2} Furthermore, we extend Eq.~\ref{detail} to a point-wise  version. To this end, we rewrite the injected random variables $z^{<L}_{q,k,v}$ as $z^{<L,1:N}_{q,k,v}=\{z^{l, i}_{q}, z^{l, i}_{k}, z^{l, i}_{v}\}_{0\leq l <L, 1\leq i\leq N}$ and we assume the points are independent. Thus, the prior and posterior distributions in Eq.~\ref{detail} can be further divided as:
\begin{equation}\label{eq10}
	\begin{split}
			&{p_\theta(z^{l, 1: N}_{q,k,v}\mid z^{<l, 1: N}_{q,k,v}, \mathcal{C}}) 
		= \Pi^{N}_{i=1}{p_\theta(z^{l, i}_{q,k,v}\mid z^{<l, 1: N}_{q,k,v}, \mathcal{C}}) \\
		&{q_\phi(z^{l, 1: N}_{q,k,v}\mid z^{<l, 1: N}_{q,k,v}, \mathcal{C}}, \mathbf{b}) 
		= \Pi^{N}_{i=1}{q_\phi(z^{l, i}_{q,k,v}\mid z^{<l, 1: N}_{q,k,v}, \mathcal{C}}, {b_i}) 
	\end{split}
\end{equation}
Also, the point-wise label prediction model can be written as:
\begin{equation}
	\begin{split}\label{eq11}
		\ln p_\theta(\mathbf{b}\mid z^{<L, 1:N}_{q,k,v}, \mathcal{C}) = \ln \Pi^N_{i=1} p_\theta({b_i}\mid z^{<L, 1:N}_{q,k,v}, \mathcal{C}) = \sum_{i=1}^{N} \ln p_\theta({b_i}\mid z^{<L, 1:N}_{q,k,v}, \mathcal{C})
	\end{split}
\end{equation}
By inserting Eq.~\ref{eq10} and Eq.~\ref{eq11} into Eq.~\ref{detail}, we can achieve the following point-wise variational lower bound:
\begin{equation}
	\begin{split}
	 \operatorname{ELBO}(\theta, \phi)	&	=  \mathbb{E}_{\Pi^{L-1}_{l=0}\Pi^{N}_{i=1}{q_\phi(z^{l, i}_{q,k,v}\mid z^{<l, 1: N}_{q,k,v}, \mathcal{C}}, {b_i})}\left[ \sum_{i=1}^{N} \ln p_\theta({b_i}\mid z^{<L, 1:N}_{q,k,v}, \mathcal{C})\right] \\
		& - \sum^{L-1}_{l=0}\mathbb{E}_{\Pi^{l}_{\tau=0}\Pi^{N}_{i=1}{q_\phi(z^{\tau, i}_{q,k,v}\mid z^{<\tau, 1: N}_{q,k,v}, \mathcal{C}}, {b_i})}\left[\ln\frac{\Pi^{N}_{i=1}{q_\phi(z^{l, i}_{q,k,v}\mid z^{<l, 1: N}_{q,k,v}, \mathcal{C}}, {b_i})}{{\Pi^{N}_{i=1}{p_\theta(z^{l, i}_{q,k,v}\mid z^{<l, 1: N}_{q,k,v}, \mathcal{C}})}}\right] \\
		= & \mathbb{E}_{\Pi^{L-1}_{l=0}\Pi^{N}_{i=1}{q_\phi(z^{l, i}_{q,k,v}\mid z^{<l, 1: N}_{q,k,v}, \mathcal{C}}, {b_i})}\left[ \sum_{i=1}^{N} \ln p_\theta({b_i}\mid z^{<L, 1:N}_{q,k,v}, \mathcal{C})\right] \\
		& - \sum^{L-1}_{l=0}\sum^{N}_{i=1}\mathbb{E}_{\Pi^{l}_{\tau=0}\Pi^{N}_{i=1}{q_\phi(z^{\tau, i}_{q,k,v}\mid z^{<\tau, 1: N}_{q,k,v}, \mathcal{C}}, {b_i})}\left[\ln\frac{{q_\phi(z^{l, i}_{q,k,v}\mid z^{<l, 1: N}_{q,k,v}, \mathcal{C}}, {b_i})}{{{p_\theta(z^{l, i}_{q,k,v}\mid z^{<l, 1: N}_{q,k,v}, \mathcal{C}})}}\right] \\
		= & \mathbb{E}_{\Pi^{L-1}_{l=0}\Pi^{N}_{i=1}{q_\phi(z^{l, i}_{q,k,v}\mid z^{<l, 1: N}_{q,k,v}, \mathcal{C}}, {b_i})}\left[ \sum_{i=1}^{N} \ln p_\theta({b_i}\mid z^{<L, 1:N}_{q,k,v}, \mathcal{C})\right] \\
		& - \sum^{L-1}_{l=0}\sum^{N}_{i=1}\mathbb{E}_{\Pi^{l-1}_{\tau=0}\Pi^{N}_{i=1}{q_\phi(z^{\tau, i}_{q,k,v}\mid z^{<\tau, 1: N}_{q,k,v}, \mathcal{C}}, {b_i})}\left[\operatorname{D_{KL}}\left({{q_\phi(z^{l, i}_{q,k,v}\mid z^{<l, 1: N}_{q,k,v}, \mathcal{C}}, {b_i})} || {{{p_\theta(z^{l, i}_{q,k,v}\mid z^{<l, 1: N}_{q,k,v}, \mathcal{C}})}}\right)\right] \\
	\end{split}
\end{equation}

Similarity, we use the deterministic hidden feature $\mathbf{h}^{l, i}_{q,k,v}$ to summarize the historical information in previous iterations:
\begin{equation}
	\begin{split}
		\operatorname{ELBO}(\theta, \phi)	&	= \mathbb{E}_{\Pi^{L-1}_{l=0}\Pi^{N}_{i=1}{q_\phi(z^{l, i}_{q,k,v}\mid \underbrace{z^{<l, 1: N}_{q,k,v}, \mathcal{C}}_{\mathbf{h}^{l, i}_{q,k,v}}}, {b_i})}\left[ \sum_{i=1}^{N} \ln p_\theta({b_i}\mid \underbrace{z^{<L, 1:N}_{q,k,v}, \mathcal{C}}_{\tilde{\mathbf{F}}^L})\right] \\
		& - \sum^{L-1}_{l=0}\sum^{N}_{i=1}\mathbb{E}_{\Pi^{l-1}_{\tau=0}\Pi^{N}_{i=1}{q_\phi(z^{\tau, i}_{q,k,v}\mid z^{<\tau, 1: N}_{q,k,v}, \mathcal{C}}, {b_i})}\Bigg[\operatorname{D_{KL}}\Bigg({{q_\phi(z^{l, i}_{q,k,v}\mid \underbrace{z^{<l, 1: N}_{q,k,v}, \mathcal{C}}_{\mathbf{h}^{l, i}_{q,k,v}}}, {b_i})} || {{{p_\theta(z^{l, i}_{q,k,v}\mid \underbrace{z^{<l, 1: N}_{q,k,v}, \mathcal{C}}_{\mathbf{h}^{l, i}_{q,k,v}}})}}\Bigg)\Bigg] \\
		= & \mathbb{E}_{\Pi^{L-1}_{l=0}\Pi^{N}_{i=1}{q_\phi(z^{l, i}_{q,k,v}\mid {\mathbf{h}^{l, i}_{q,k,v}}}, {b_i})}\left[ \sum_{i=1}^{N} \ln p_\theta({b_i}\mid {\tilde{\mathbf{F}}^L})\right] \\
		& - \sum^{L-1}_{l=0}\sum^{N}_{i=1}\mathbb{E}_{\Pi^{l-1}_{\tau=0}\Pi^{N}_{i=1}{q_\phi(z^{\tau, i}_{q,k,v}\mid z^{<\tau, 1: N}_{q,k,v}, \mathcal{C}}, {b_i})}\Bigg[\operatorname{D_{KL}}\Bigg({{q_\phi(z^{l, i}_{q,k,v}\mid {\mathbf{h}^{l, i}_{q,k,v}}}, {b_i})} || {{{p_\theta(z^{l, i}_{q,k,v}\mid {\mathbf{h}^{l, i}_{q,k,v}}})}}\Bigg)\Bigg] \\
	\end{split}
\end{equation}

\section{Proof of Theorem 1} \label{proof2}
We let $\{{}^{(\kappa)}\mathcal{M}^{sac}_i\}_{i=1}^J$ be the randomly sampled \textit{hypothetical inlier} subset in RANSAC, and let  $\mathcal{C}_{in}$, $\mathcal{C}_{out}$ and $p_{in}$ be the inlier subset, outlier subset and the inlier ratio, respectively, $\mathcal{C}=\mathcal{C}_{in}\cup\mathcal{C}_{out}$, $p_{in}={|\mathcal{C}_{in}|}/{|\mathcal{C}|}$. 
We also denote the inliers in seed subset $\mathcal{C}_{seed}$ as $\tilde{\mathcal{C}}_{in}=\mathcal{C}_{in}\cap\mathcal{C}_{seed}$. Then, we can derive the following theorem: 
\begin{theorem}  \label{theorem1}
	Assume the number of outliers in ${}^{(\kappa)}\tilde{\mathcal{M}}_{\mathbf{c}_i}$ ($\mathbf{c}_i\in\tilde{\mathcal{C}}_{in}$) follows a Poisson distribution $Pois\left(\alpha \cdot \kappa\right)$. Then, if $\alpha < -\frac{1}{\kappa} \cdot \log \left[1 - (1 - p_{in}^\kappa)^{J/|\tilde{\mathcal{C}}_{in}|}\right]\triangleq\mathcal{U}$, the probability of our method achieving the inlier subset is greater than or equal to that of RANSAC. 
	\begin{equation}
		\begin{split}\label{neq1}
			&P\Big(\max_{\mathbf{c}_i \in \mathcal{C}_{seed}} |{}^{(\kappa)}\tilde{\mathcal{M}}_{\mathbf{c}_i}\cap \mathcal{C}_{in}| =\kappa\Big) \geq \\ &P\Big(\max_{1\leq i\leq J} |{}^{(\kappa)}\mathcal{M}^{sac}_i\cap \mathcal{C}_{in}| =\kappa\Big).
		\end{split}
	\end{equation}
\end{theorem} 
\begin{proof}  The probability that RANSAC can achieve the inlier subset can be calculated via: 
	\begin{equation}
		\begin{split}
			P\Big(\max_{1\leq i\leq J} |{}^{(\kappa)}\mathcal{M}^{sac}_i\cap \mathcal{C}_{in}| =\kappa\Big)
			& = 1 -  P\Big(\max_{1\leq i\leq J} |{}^{(\kappa)}\mathcal{M}^{sac}_i\cap \mathcal{C}_{in}| <\kappa\Big) \\
			& =1-P\Big( |{}^{(\kappa)}\mathcal{M}^{sac}_1\cap \mathcal{C}_{in}| <\kappa\Big)\cdots P\Big( |{}^{(\kappa)}\mathcal{M}^{sac}_J\cap \mathcal{C}_{in}| <\kappa\Big) \\
			& \overset{(1)}{=} 1 - P\Big( |{}^{(\kappa)}\mathcal{M}^{sac}_l\cap \mathcal{C}_{in}| <\kappa\Big)^J\\
			& = 1 - \Big(1-P\Big( |{}^{(\kappa)}\mathcal{M}^{sac}_l\cap \mathcal{C}_{in}| =\kappa\Big)\Big)^J\\
			& = 1 - \left(1 - \frac{C^{\kappa}_{|\mathcal{C}_{in}|}}{C^\kappa_{|\mathcal{C}|}}\right)^J \\
			& = 1 - \left(1 - \frac{|\mathcal{C}_{in}|\cdots (|\mathcal{C}_{in}|-\kappa+1)}{|\mathcal{C}| \cdots (|\mathcal{C}|-\kappa+1)}\right)^J \\
			& \leq 1 - \left(1 - \left(\frac{|\mathcal{C}_{in}|}{|\mathcal{C}|}\right)^\kappa\right)^J =  1 - \left(1 - p_{in}^\kappa\right)^J,
		\end{split}
	\end{equation}
	where step (1) is based on that random variables $\{|{}^{(\kappa)}\mathcal{M}^{sac}_l\cap \mathcal{C}_{in}| <\kappa\}_{1\leq l \leq J}$ are i.i.d. Analogously, the probability of our method achieving the inlier subset can be calculated via: 
	\begin{equation}
		\begin{split}
			P\Big(\max_{\mathbf{c}_i\in \mathcal{C}_{seed}} |{}^{(\kappa)}\tilde{\mathcal{M}}_{\mathbf{c}_i}\cap \mathcal{C}_{in}| =\kappa\Big)
			& \geq P\Big(\max_{\mathbf{c}_i\in \tilde{\mathcal{C}}_{in}} |{}^{(\kappa)}\tilde{\mathcal{M}}_{\mathbf{c}_i}\cap \mathcal{C}_{in}| =\kappa\Big) \\
			& = 1 -  P\Big(\max_{\mathbf{c}_i\in \tilde{\mathcal{C}}_{in}} |{}^{(\kappa)}\tilde{\mathcal{M}}_{\mathbf{c}_i}\cap \mathcal{C}_{in}| < \kappa \Big) \\
			& {=} 1 -  \Pi_{\mathbf{c}_i \in\tilde{\mathcal{C}}_{in}}P\Big( |{}^{(\kappa)}\tilde{\mathcal{M}}_{\mathbf{c}_i}\cap \mathcal{C}_{in}| < \kappa\Big) \\
			& \overset{(1)}{=} 1 -  P\Big( |{}^{(\kappa)}\tilde{\mathcal{M}}_{\mathbf{c}_l}\cap \mathcal{C}_{in}| < \kappa\mid \mathbf{c}_l\in \tilde{\mathcal{C}}_{in}\Big)^{|\tilde{\mathcal{C}}_{in}|} \\
			& = 1 -  \Big(1-P\Big( |{}^{(\kappa)}\tilde{\mathcal{M}}_{\mathbf{c}_l}\cap \mathcal{C}_{in}| = \kappa\mid \mathbf{c}_l\in \tilde{\mathcal{C}}_{in}\Big)\Big)^{|\tilde{\mathcal{C}}_{in}|} \\
			& = 1 -  \Big(1-P\Big( |{}^{(\kappa)}\tilde{\mathcal{M}}_{\mathbf{c}_l}\cap \mathcal{C}_{out}| = 0\mid \mathbf{c}_l\in \tilde{\mathcal{C}}_{in}\Big)\Big)^{|\tilde{\mathcal{C}}_{in}|} \\
			& \overset{(2)}{=} 1 - \left(1 - e^{-\alpha\cdot \kappa}\right)^{|\tilde{\mathcal{C}}_{in}|},
		\end{split}
	\end{equation}
where step (1) is based on that random variables $\{|{}^{(\kappa)}\tilde{\mathcal{M}}_{\mathbf{c}_i}\cap \mathcal{C}_{in}| < \kappa\}_{\mathbf{c}_i\in\tilde{\mathcal{C}}_{in}}$ are i.i.d; Step (2) is based on our assumption of Poisson distribution: $P\Big( |{}^{(\kappa)}\tilde{\mathcal{M}}_{\mathbf{c}_i}\cap \mathcal{C}_{out}| = m\mid \mathbf{c}_i\in \tilde{\mathcal{C}}_{in}\Big)=\frac{(\alpha\cdot\kappa)^m e^{-\alpha\cdot\kappa}}{m!}$. Finally, we let $1 - \left(1 - e^{-\alpha\cdot \kappa}\right)^{|\tilde{\mathcal{C}}_{in}|} \geq 1 - \left(1 - p_{in}^\kappa\right)^J $ and can get that if $\alpha \leq -\frac{1}{\kappa} \cdot \log \left[1 - (1 - p_{in}^\kappa)^{J/|\tilde{\mathcal{C}}_{in}|}\right]$, the inequality~\ref{neq1} holds. 
\end{proof}

\section{Qualitative Evaluation}
We first give some qualitative comparisons with PointDSC~\cite{bai2021pointdsc} (our baseline) on 3DLoMatch benchmark dataset in Fig.~\ref{graph}. 
As can be observed, in cases containing extremely low-overlapped regions (red box), our method can achieve more precise alignment. 
Those mainly benefit from our more discriminative correspondence embedding based on variational non-local network for more reliable inlier clustering. 
Also, we visualize the registration results on KITTI dataset in Fig.~\ref{graph2}.

\begin{figure}[ht]
	\centering
	\includegraphics[width=\columnwidth]{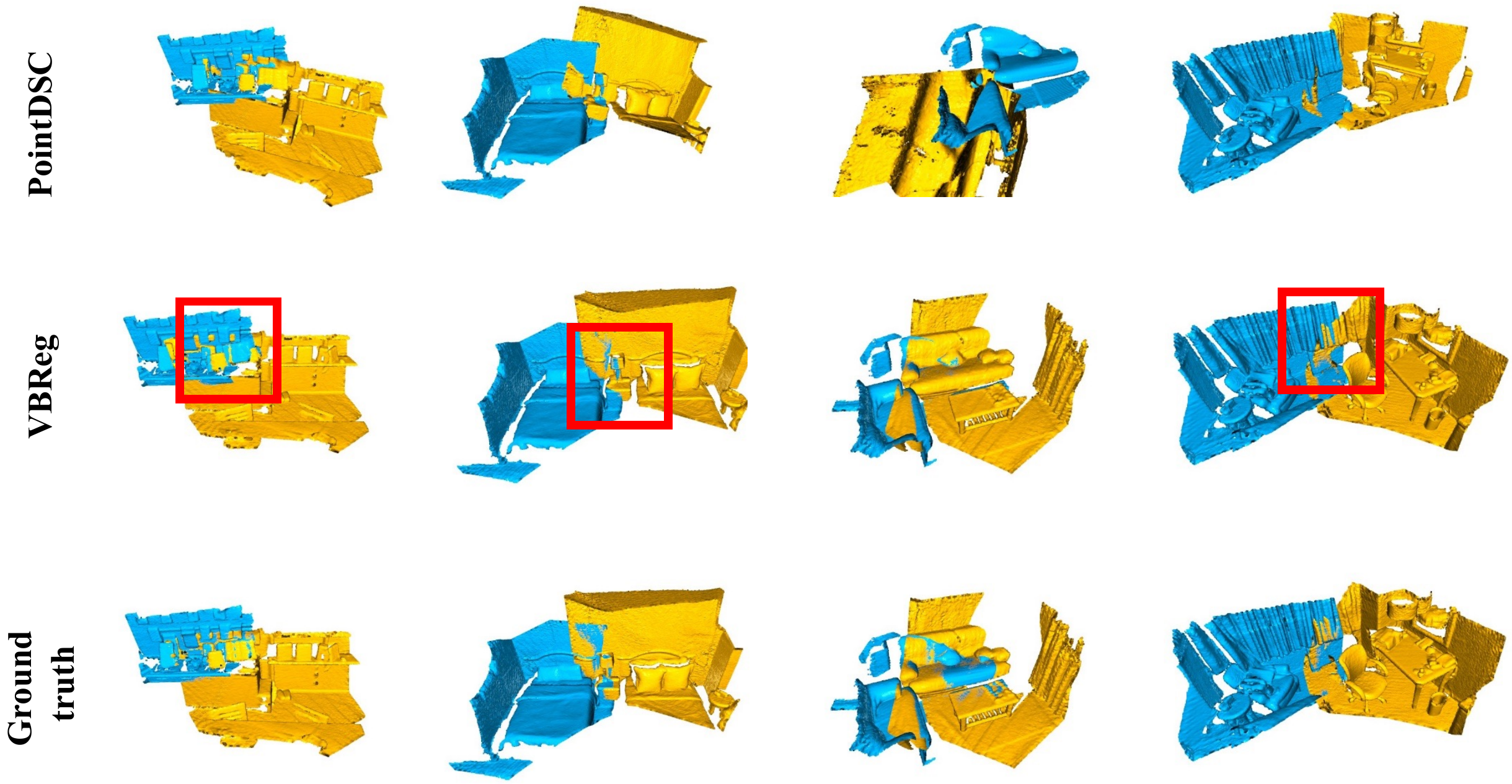}
	\caption{Qualitative comparison with PointDSC~\cite{bai2021pointdsc} (baseline) on {3DLoMatch} benchmark~\cite{huang2021predator}.} \label{fig:model}
	\label{graph}
\end{figure}

\begin{figure}[ht]
	\centering
	\includegraphics[width=\columnwidth]{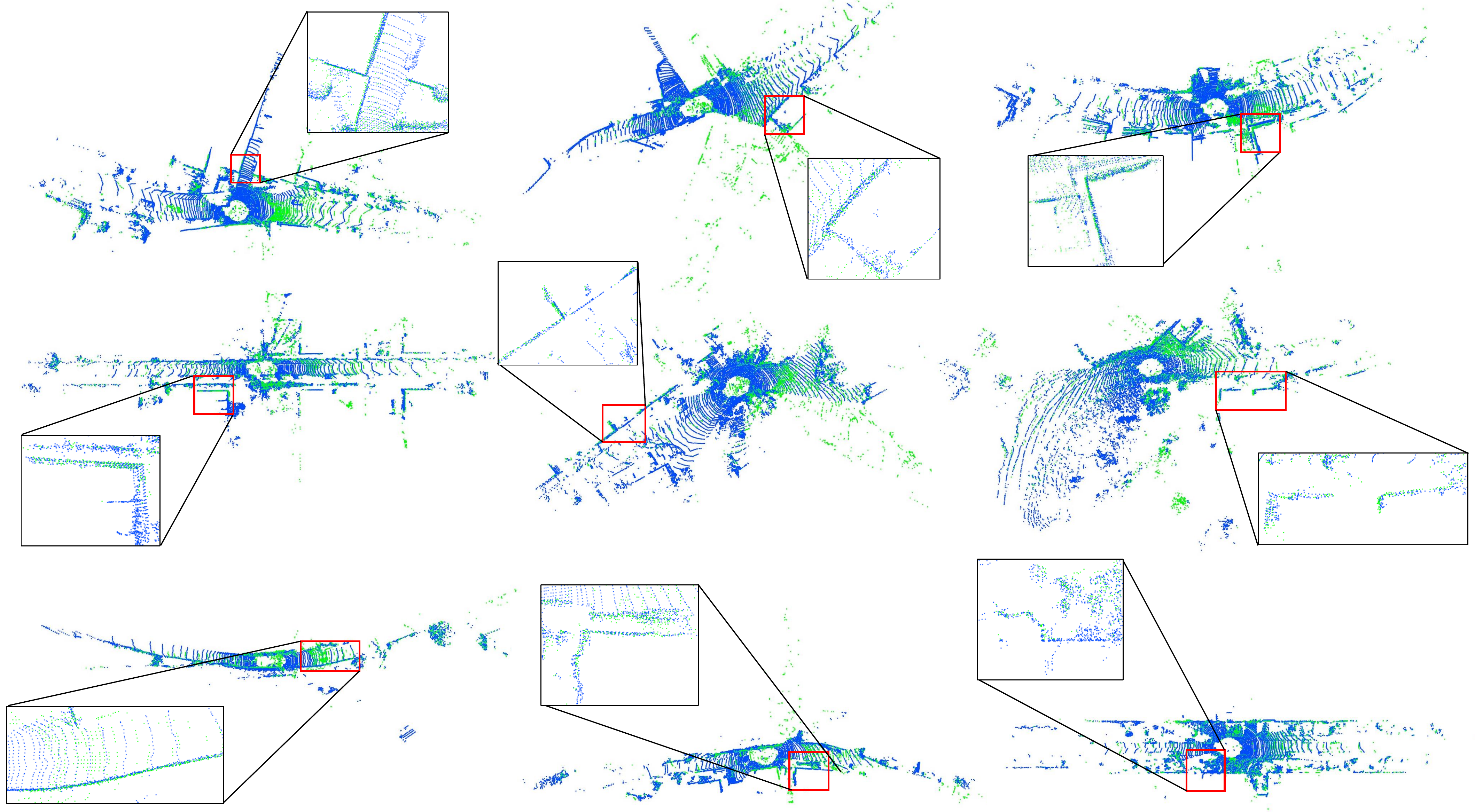}
	\caption{Registration visualization on KITTI benchmark~\cite{geiger2012we}. } \label{fig:model}
	\label{graph2}
\end{figure}

\end{document}